%% file: paper.tex
\documentclass[a4paper,11pt]{article}

\usepackage{authblk} 
\usepackage{fullpage}
\usepackage[hyphens]{url}  
\usepackage{graphicx} 
\usepackage{caption} 

\input{preamble}

\begin{document}

\title{Diversity in Kemeny Rank Aggregation: \\ A Parameterized Approach }
\author[1]{Emmanuel Arrighi}
\author[2]{Henning Fernau} 
\author[3]{Daniel Lokshtanov}
\author[1]{\\ Mateus de Oliveira Oliveira} 
\author[2]{Petra Wolf}
\affil[1]{University of Bergen, Norway}
\affil[1]{{\small\texttt{\{emmanuel.arrighi, mateus.oliveira\}@uib.no}}}
\affil[2]{University of Trier, Germany}
\affil[2]{{\small\texttt{\{fernau, wolfp\}@informatik.uni-trier.de}}}
\affil[3]{University of California Santa Barbara, CA, USA}
\affil[3]{{\small \texttt{daniello@ucsb.edu}}}

\maketitle
\input{00-abstract}

\input{01-introduction}
\input{02-preliminaries}

\input{03-KRA}

\input{04-CO}
\input{05-DPPCO}
\input{06-SubexponentialPCOSafe}

\input{07-Conclusion}

\input{08-Acknowledgements}

\bibliographystyle{abbrv}
\bibliography{bib}

\end{document}

%% file: preamble.tex
\usepackage{todonotes}

\usepackage[absolute]{textpos}

\usepackage{amsfonts, amsopn, amsmath, amsthm}
\usepackage{enumitem}
\usepackage{comment}

\newtheorem{theorem}{Theorem}
\newtheorem{definition}[theorem]{Definition}
\newtheorem{problem}[theorem]{Problem}
\newtheorem{corollary}[theorem]{Corollary}
\newtheorem{lemma}[theorem]{Lemma}

\newtheorem{remark}[theorem]{Remark}
\newtheorem{claim}[theorem]{Claim}

\newtheoremstyle{named}{}{}{\itshape}{}{\bfseries}{.}{.5em}{Restatement of #1 \thmnote{#3 }}
\theoremstyle{named}

\newcommand{\rankingset}{R}

\newcommand{\intv}[1]{\left [ #1 \right ]}

\newcommand{\sm}{\setminus}

\newcommand{\edges}{E}

\newcommand{\diversity}{\mathrm{Div}}

\newcommand{\ktdiversity}{\operatorname{\mathrm{KT-Div}}}

\newcommand{\forgotten}{\mathrm{forg}}
\newcommand{\introduced}{\mathrm{intro}}

\usepackage{amsthm,amssymb,amsmath}

\usepackage{fancybox}

\renewcommand{\problem}[3]{\begin{center}\shadowbox{\begin{minipage}{.98\linewidth}
				{{{\bf Problem name:} {\sc #1}{\index{\sc #1}}} 
				\\
				{\bf Given:} #2\\
				{\bf Output:} #3}
			\end{minipage}}
	\end{center}}

\newcommand{\partialorder}{\rho}
\newcommand{\linearorder}{\tau}
\newcommand{\baseset}{\candidates}
\newcommand{\KTd}{\operatorname{KT-dist}}

\newcommand{\cocomparabilitygraph}[1]{G_{#1}}
\newcommand{\graph}{G}
\newcommand{\neighbours}{N}

\newcommand{\pathdec}{\mathcal{P}}
\newcommand{\lengthdec}{l}
\newcommand{\tw}{\mathit{tw}}
\newcommand{\tree}{\mathcal{T}}

\newcommand{\parameter}{k}

\newcommand{\Oh}{{\cal{O}}}

\newcommand{\NP}{\textsf{NP}}
\newcommand{\FPT}{\textsf{FPT}}
\newcommand{\unanimityorder}{\rho}
\newcommand{\solutionset}{S}

\newcommand{\pairs}{\mathbb{P}}
\newcommand{\optimalfunction}{f}

\newcommand{\aproject}{p}

\newcommand{\distanceoptimality}{\delta}

\newcommand{\pbKRAname}{Kemeny Rank Aggregation}
\newcommand{\pbGKRAname}{Kemeny Rank Aggregation}
\newcommand{\pbDivGKRAname}{Diverse Kemeny Rank Aggregation}
\newcommand{\pbPCOname}{Positive Completion of an Ordering}
\newcommand{\pbCOname}{Completion of an Ordering}
\newcommand{\pbDivCOname}{Diverse Completion of an Ordering}

\newcommand{\KRAname}{\textsc{\pbKRAname}}
\newcommand{\PCOname}{\textsc{\pbPCOname}}
\newcommand{\COname}{\textsc{\pbCOname}}

\newcommand{\N}{\mathbb{N}}
\newcommand{\naturalIntegers}{\N}

\newcommand{\pw}{\mathit{pw}}
\newcommand{\cpw}{\mathit{cpw}}

\newcommand{\votes}{\Pi}
\newcommand{\vote}{\pi}
\newcommand{\candidates}{C}
\newcommand{\candidate}{c}
\newcommand{\mvotes}{m}
\newcommand{\ncandidates}{n}
\newcommand{\compatiblesequence}{\gamma}
\newcommand{\diversitycompatiblesequence}{\hat{\compatiblesequence}}

\newcommand{\partialcost}{\gamma}
\DeclareMathOperator{\cost}{\mathfrak{c}}
\DeclareMathOperator{\width}{w}
\DeclareMathOperator{\opt}{opt}

\newcommand{\defeq}{\;\dot{=}\;}
\newcommand{\apartialorder}{\rho}
\newcommand{\astrictorder}{\sigma}

\newcommand{\intervalorder}{\iota}

\newcommand{\alinearorder}{\tau}

\newcommand{\xelement}{x}
\newcommand{\yelement}{y}

\newcommand{\agraph}{G}

\newcommand{\vertexset}{V}
\newcommand{\leftset}{L}

\newcommand{\edgeset}{E}
\newcommand{\pathdecomposition}{D}
\newcommand{\bag}{B}
\newcommand{\lengthdecomposition}{l}
\newcommand{\vvertex}{v}
\newcommand{\uvertex}{u}

\newcommand{\transitiveclosure}{\mathbf{tc}}
\newcommand{\symmetricclosure}{\mathbf{sc}}

\newcommand{\diversityparameter}{d}
\newcommand{\mindiversityparameter}{s}
\newcommand{\scatterednessparameter}{\mindiversityparameter}

\newcommand{\partialdiversity}{\partial}
\newcommand{\partialscatteredness}{\xi}
\newcommand{\increasepair}{\Delta\partialdiversity}
\newcommand{\increasescatteredness}{\Delta\partialscatteredness}
\newcommand{\numbersolutions}{r}

\newcommand{\longversion}[1]{}
\newcommand{\aposition}{p}
\newcommand{\triples}{\mathcal{T}}
\newcommand{\promisingtriples}{\mathcal{Q}}
\newcommand{\bigtuples}{\mathcal{I}}
\newcommand{\abigtuple}{\frak{u}}

\newcommand{\asubsetbag}{S}

\newcommand{\acluster}{C}

%% file: 00-abstract.tex
\begin{abstract}
In its most traditional setting, the main concern of optimization theory is the search for optimal solutions 
for instances of a given computational problem. A recent trend of research in artificial intelligence,
called \emph{solution diversity}, has focused on the development of notions of optimality that
may be more appropriate in settings where subjectivity is essential. The idea is that instead of aiming at
the development of algorithms that output a \emph{single} optimal solution, the goal 
is to investigate algorithms that output a small set of \emph{sufficiently good} solutions 
that are \emph{sufficiently diverse} from one another. In this way, the user 
has the opportunity to choose the solution that is most appropriate to the context
at hand. It also displays the richness of the solution space. 
	
When combined with techniques from parameterized complexity theory, the paradigm of diversity of solutions
offers a powerful algorithmic framework to address problems of practical relevance. In this work, we 
investigate the impact of this combination in the field of Kemeny Rank Aggregation, a well-studied class of 
problems lying in the intersection of order theory and social choice theory and also in the field of order
theory itself. In particular, we show 
that the Kemeny Rank Aggregation problem is fixed-parameter tractable with respect to natural parameters providing natural formalizations of
the notions of \emph{diversity} and of the notion of a \emph{sufficiently good} solution. Our main 
results work both when considering the traditional setting of aggregation over linearly ordered votes, and 
in the more general setting where votes are partially ordered. \\
\end{abstract}

%% file: 01-introduction.tex
\section{Introduction}
\label{section:Introduction}

Traditionally, in optimization theory, when given an instance of a
computational problem, one is interested in computing
{\em some} optimal solution for the instance in question. 
For certain problems of practical relevance, this 
framework may not be satisfactory because it 
precludes the user from the possibility of choosing among optimal
solutions in case more than one exists, or even from choosing a 
slightly less optimal solution that may be a better fit for 
the intended application, due to subjective factors. 

A recent up-coming trend of research in artificial intelligence,
called \emph{diversity of solutions}~\cite{PetitTrapp2019,BasteFJMOPR20,ingmar2020modelling,BasteEtAl19,fomin2020diverse},
has focused on the development of notions of optimality that may be more appropriate 
in settings where subjectivity is essential. The idea is that instead of aiming at
the development of algorithms that output a \emph{single} optimal solution, the goal 
is to investigate algorithms that output a small set of {\em sufficiently good} solutions 
that are \emph{sufficiently diverse} from one another. In this way, the user 
has the opportunity to choose the solution that is most appropriate to the context at hand.
The intuition is that the criteria employed by the user to decide what an appropriate solution
is may be subjective, and therefore, impractical or even impossible to be formalized at the level 
of the problem specification. Examples of such criteria are aesthetic, economic, political, 
environmental criteria. Another motivation comes from the problem of finding several good committees such that each committee member
is not overloaded with these commitments, as described in \cite{BreKacNie2020}; again, some diversity could be helpful.

One source of difficulty when trying to develop efficient algorithms for diverse variants of 
computational problems is the fact that these problems may be computationally hard. In particular,
many interesting computational problems that are suitable for being studied from the 
perspective of diversity of solutions are already \NP-hard in the usual variant in which one 
asks for a single solution. Additionally, it may be the case that even problems that are 
polynomial-time solvable in the single-solution version  become \NP-hard when
considering diverse sets of solutions. One way to circumvent this difficulty is 
to combine the framework of diversity of solutions with the framework of fixed-parameter tractability theory \cite{DowneyFellows99}. A central notion in this framework 
is the one of fixed-parameter tractability. An algorithm for a given computational problem
is said to be fixed-parameter tractable with respect to parameters $k_1,\dots,k_r$ if it
runs in time $f(k_1,\dots,k_r)\cdot n^{O(1)}$, where $n$ is the size of the input and 
$f$ is a computable function that depends only on the parameters. The intuition is that
if the range of the parameters is small on instances of practical relevance, then even if 
the function $f$ grows relatively fast, the algorithm can be considered to be fast enough 
{\em for practical purposes}.

When studying a given computational problem from the point of veiw of solution diversity, it is crucial
to have in hands a notion of distance between solutions for that problem. The {\em diversity} of a set of 
solutions $\solutionset$ is then defined as the sum of distances between pairs of solutions in $\solutionset$. 
We denote this measure by $\diversityparameter$. Intuitively, diversity is
a global measure for how representative a set of solutions is among the space of solutions. 
Three natural parameters can be used to quantify how good a diverse set of solutions is: the number $\numbersolutions$ 
of solutions in the set,  the maximum distance $\distanceoptimality$ between the cost of a solution in the set and the cost 
of an optimal solution (we call this parameter the {\em solution imperfection} of the set), and the minimum {\em required}
distance $\mindiversityparameter$ between any two solutions in the set. This last parameter is also known in the 
literature as the {\em scatteredness} of $\solutionset$ \cite{DBLP:journals/bit/Galle89}.
Intuitively, the parameter $\numbersolutions$ is expected to be small because in practical applications we do not 
want to overwhelm the user with an excessive number of choices.
The parameter $\distanceoptimality$ is expected to be small because while we do want $\solutionset$ to be diverse, 
we do not want to allow solutions of bad quality. 
Finally, in the context of our work, solution diversity is formalized by the parameter $d$, and we only use the 
scatteredness parameter $s$ to enforce that one cannot achieve high diversity by copying a given solution 
an arbitrary number of times. For this it is enough to require $s=1$. 
It is worth noting that it is possible to have $\mindiversityparameter$ very small (say $s=1$),
but $\diversityparameter$ very large, since some pairs of solutions in the set may be very far apart from one another. 

In this work, we investigate the impact of the notions of diversity of solutions 
and of fixed parameter tractability theory in the context of social choice theory.  
In particular, we focus on the framework of preference list aggregation introduced by 
Kemeny in the late fifties \cite{Kemeny59}. Intuitively, preference lists are a formalism
used in social choice theory to capture information about choice in applications involving the selection of candidates, products, etc.\ 
by a group of voters. The task is then to find a ranking of the candidates that maximizes 
the overall satisfaction among the voters. This problem is commonly referred to in modern terminology as 
the {\em Kemeny rank aggregation} (\textsc{KRA}) problem.   In its most general setting, the ranking 
cast by each voter is a partial order on the set of candidates.  
The distance measure we use to define our diverse version for \textsc{KRA} is the 
Kendall-Tau distance which is widely used in the context of preference aggregation.\footnote{To cite from the Stanford Encyclopedia of Philosophy~\cite{sep-social-choice}: \em At the heart of social choice theory is the analysis of preference aggregation, understood as the aggregation of several individuals' preference rankings of two or more social alternatives into a single, collective preference ranking (or choice) over these alternatives.}
 Its popularity is underlined by articles describing these issues for the interested public audience; see \cite{FarTim2019}.

\subsection{Our Contribution}

Our main result is a multiparametric algorithm for {\sc Diverse KRA} over partially ordered votes that runs in time 
$f(\width,\numbersolutions,\distanceoptimality,\scatterednessparameter)\cdot \diversityparameter\cdot n
\cdot \log (n^2 \cdot m)$ where $n$ is the number of candidates, $m$ is the number of votes,
$\numbersolutions,\distanceoptimality,\scatterednessparameter$ and $\diversityparameter$
are the parameters discussed above, and $\width$ is the \emph{unanimity width} of the votes. That is to say, 
the pathwidth of the cocomparability graph of the unanimity order of the input votes (Corollary~\ref{corollary:DiverseGKRA}). Intuitively, this width measure is a quantification of the amount of disagreement between the votes.
Note that pathwidth and treewidth coincide for the class of cocomparability graphs~\cite{habib1994treewidth}.

On the path towards obtaining our results for Kemeny Rank Aggregation, we also make contributions to 
problems of independent interest arising in the theory of cocomparability graphs.
First, by leveraging on classic results from \cite{habib1994treewidth}, we show that
the problem of constructing a $\partialorder$-consistent path decomposition of approximately minimum
width for the cocomparability graph~$\agraph_{\partialorder}$ of a given partial-order~$\partialorder$ is 
fixed-parameter tractable with respect to the pathwidth of~$\agraph_{\partialorder}$. While it was known 
that the pathwidth and the $\partialorder$-consistent pathwidth of $\agraph_{\apartialorder}$ are always the
same~\cite{AFOW2020}, and that there were fixed-parameter tractable algorithms for computing path decompositions 
of approximately minimum width due to structural properties of cocomparability graphs~\cite{DBLP:journals/dam/BouchitteKMT04}, the problem of computing such a decomposition satisfying the additional 
$\partialorder$-consistent requirement was open~\cite{AFOW2020}.

Second, we note that the notion of Kendall-Tau distance between partial orders (formally defined in Section \ref{section:KRA}), which is used to define our notion of diversity, can
be applied equally well in the more general context of the \textsc{\pbCOname} problem (\textsc{CO}), a problem 
of fundamental importance in order theory that unifies several problems of relevance for artificial intelligence, 
such as \textsc{KRA}, \textsc{One-Sided Crossing Minimization} (an important sub-routine used in the search for
good hierarchical representations of graphs), and \textsc{Grouping by Swapping} (a relevant problem in the field of memory management)~\cite{DBLP:journals/algorithmica/WongR91}.
For a matter of generality, we first develop a
$f(\width,\numbersolutions,\distanceoptimality,\scatterednessparameter)\cdot \diversityparameter\cdot n
\cdot \log (n^2 \cdot m)$ time
algorithm for \textsc{Diverse CO} (Theorem~\ref{theorem:SolutionConstructionDiverse}) and then obtain our 
main result for {\sc Diverse KRA} as a corollary. In the more general context of CO, the parameter $\width$ is 
the width of the cocomparability graph of the partial order given at the input. 

Finally, building on recent advances in the theory of $C_k$-free graphs \cite{DBLP:conf/soda/ChudnovskyPPT20} we establish
an upper bound for the pathwidth of a cocomparability graph in terms of the number of edges of the graph. As a by-product of 
this result, we obtain the first algorithm running in time $\Oh^*(2^{\Oh(\sqrt{k})})$ (Theorem \ref{theorem:main}) 
for the positive completion of an ordering problem (PCO), a special case of CO which still generalizes KRA and other important 
combinatorial problems. Previous to our work, the best algorithm for this problem parameterized by cost had asymptotic time complexity of 
$\Oh^*(k^{\Oh(\sqrt{k})}) = \Oh^{*}(2^{\sqrt{k}\log k})$. Therefore, we remove the log-factor in the exponent.
According to Theorem~18 in~\cite{AFOW2020}, this is optimal under the Exponential Time Hypothesis (ETH).

It is worth noting that in the context of {\sc KRA} over totally ordered votes, the existence of 
diverse sets of high-quality solutions implies that any optimal solution disagrees significantly 
with some of the voters. More precisely, let $k_{opt}$ be the cost of an optimal solution and suppose that there are two 
solutions with cost $k_{\opt} + \distanceoptimality_1$ and $k_{opt}+\distanceoptimality_2$, respectively. 
It is possible to show that $\max\{k_{\mathit{opt}} + \distanceoptimality_1,k_{\mathit{opt}} + \distanceoptimality_2\}$ 
is at least half the number of votes. Therefore, if two solutions have small solution imperfection, 
then $k_{\mathit{opt}}$ is large. In the other direction, if there is a strong consensus 
among the voters ($k_{\mathit{opt}}$ is small) then $\distanceoptimality_1$ or $\distanceoptimality_2$ must be large. 
Intuitively, in the context of aggregation over totally-ordered votes, the more disagreement there is between an optimal
ranking and the ranking provided by the voters, the more one can benefit from the framework of solution diversity. 
In the context of aggregation over partially ordered votes, such a correlation between solution imperfection and 
optimality does not necessarily hold even for constant unanimity width. For instance, consider a set 
of partially ordered votes where the unanimity order is a bucket order with buckets of size $2$ (i.e., unanimity width equal to~$1$).
Then depending on the instance, we can have diverse sets of solutions with solution imperfection $0$ and optimal cost $0$. 
Therefore, in the context of partially ordered votes, the notion of diversity makes sense 
even in the case where voters have small disagreement between each other.

\subsection{Related Work}

The framework of diversity of solutions, under distinct notions of diversity, 
has found applications in several subfields of artificial intelligence, such as
information search and retrieval~\cite{GollapudiSharma2009,AbbassiMirrokniThakur2013},
mixed integer programming~\cite{GloverEtAl2000,DannaWoodruff2009,PetitT15}, 
binary integer linear programming~\cite{GreistorferEtAl2008,TrappKonrad2015},
constraint programming~\cite{HebrardHOW05,HebrardEtAl2007}, 
SAT solving~\cite{Nadel2011}, recommender systems~\cite{AdomaviciusKwon2014},
routing problems~\cite{SchittekatS2009}, answer set programming~\cite{EiterEtAl2013},
decision support systems~\cite{LokketangenW2005,HadzicEtAl2009}, genetic algorithms~\cite{GaborBPS18,WinebergO03},
planning~\cite{BasteEtAl19}, 
and in many other fields. Recently, a general framework for addressing diversity of 
solutions from the perspective of parameterized complexity theory was 
developed  \cite{BasteFJMOPR20}. This framework allows one to
convert dynamic programming algorithms for finding an optimal solution for instances
of a given problem into dynamic programming algorithms for finding a small set
of \longversion{sufficiently }diverse solutions. 

Notice that there is also the related area of enumerating all optimal
solutions, or at least encoding them all; this is known as knowledge
compilation in artificial intelligence, see, e.g.,
\cite{DBLP:journals/jair/DarwicheM02,DBLP:journals/ai/FargierM14,DBLP:conf/ijcai/Marquis11}.
These types of questions have also been considered from a more combinatorial
viewpoint; confer
\cite{DBLP:journals/algorithmica/FominHKPV14,DBLP:journals/algorithmica/GolovachHKKSV18,DBLP:journals/siamdm/KanteLMN14,DBLP:reference/algo/KanteN16}.
But from a practical perspective, it is not really desirable to confront a user
with an exponential number of different solutions, but she wants to know what
the real alternatives are.

Two measures of diversity of a set $\solutionset$ of solutions have been particularly explored in the literature.
The first one is the sum of distances between pairs of solutions in $\solutionset$. The 
second one is the minimum distance $\scatterednessparameter$ between any two solutions in $\solutionset$. 
This last notion has been also known in the literature as \emph{scatteredness} \cite{DBLP:journals/bit/Galle89}.
Both notions have been used in the context of vertex- and edge-problems 
on graphs using the Hamming distance of solutions as the distance measure \cite{BasteFJMOPR20,GaborBPS18,WinebergO03,fomin2020diverse}.
It is worth noting that when used alone, the diversity measure defined as the sum of Hamming distances 
has some weaknesses. For instance, if we take a pair $\{A,B\}$ of solutions of diversity $d$, then
the list $A_1,A_2,...,A_r, B_1,B_2,...,B_r$, where each $A_i$ is a copy of~$A$, and each $B_i$ is a copy of~$B$,
has high diversity ($d'> r^2\cdot d$), while this list clearly opposes the intuitive notion of a diverse set of solutions. 
This weakness can be significantly mitigated by considering diversity in conjunction with scatteredness. 
For instance, by setting $\scatterednessparameter \geq 1$, we already guarantee that all elements 
in a list of solutions will be distinct from each other. 

The notions of Kemeny score and Kemeny rank aggregation were fist introduced in \cite{Kemeny59}. 
These notions play a crucial role in fields such as social choice theory \cite{BarTovTri89,Kla2004} and found applications in diverse areas as
 information retrieval~\cite{DBLP:conf/www/DworkKNS01}, preference learning \cite{DBLP:conf/alt/ClemenconKS18}, genetic map generation \cite{DBLP:journals/tcbb/JacksonSA08}, etc. 
We also point to more recent publications like \cite{CorGalSpa2013,GilPorSpa2020} in the AI area that also allow to follow the literature on this topic. 
As shown by Young and Levenglick~\cite{YouLev78}, this is the only 
aggregation method satisfying three natural requirements: \emph{symmetry}, \emph{consistency} and being \emph{Condorcet}.
This notion can also be regarded as a maximum-likelihood estimator \cite{young1988condorcet}. 
While the \textsc{KRA} method has been originally designed to deal with the setting 
where the vote cast by each voter is a total order over the set of candidates,
this method has been naturally generalized to the context where the ranking 
cast by each voter is a partial order over the set of candidates. 
Such partially ordered votes already generalize the notion of Bucket orderings (also known as weak orders) which has been 
studied in the context of aggregation and can be viewed as looking at rankings with ties, also called indifference classes \cite{DBLP:journals/mss/Hudry12,Zwicker18}.

Most relevant to our studies is reference~\cite{Zwicker18}. There, Kemeny rank aggregation is generalized both to starting with and aiming at bucket orderings, also known as weak orders. More precisely, the author considers the problem to find a bucket order with at most $k$ buckets (indifference classes) that compromises a heap of bucket orders, each of which has at most $j$ buckets. It is shown that this variation of KRA is NP-complete even for small constant values of $j,k$. This makes this parameterization useless from the viewpoint of parameterized complexity. However, one could also reverse the requirement and ask the voter to create a ranking  in which no more than $\ell$ candidates are mutually indifferent; in other words, the buckets have a capacity limited by $\ell$. As we will see, this turns out to be a useful parameter and we will study it in this paper. 

The \emph{minimum} Kendall-Tau distance over all pairs in a set is analogous 
to the minimum Hamming distance among pairs of solutions used in \cite{HebrardHOW05,HebrardEtAl2007}.

%% file: 02-preliminaries.tex
\section{Preliminaries}
\label{section:Preliminaries}

If $n$ is a positive integer, $\intv{n}=\{1,\ldots,n\}$ denotes the discrete interval of the first $n$ positive integers, and $\intv{n}_0=\intv{n}\cup\{0\}$.
$ \N$ denotes the non-negative integers.

Let $\baseset$ be a set. A \emph{partial order} over $\baseset$ is a reflexive,
antisymmetric and transitive binary relation $\partialorder \subseteq \baseset \times \baseset$. We say that $\partialorder$ is a \emph{linear order} if additionally,
for each $(x,y) \in \baseset \times \baseset$, either $(x,y) \in \partialorder$ or $(y,x) \in \partialorder$. The \emph{comparability relation} $\symmetricclosure(\partialorder)$ of $\partialorder$ is the symmetric closure of~$\partialorder$, i.e., $(x,y)\in \symmetricclosure(\partialorder)$ iff $(x,y)\in\partialorder$ or  $(y,x) \in \partialorder$. For instance, $\symmetricclosure(\partialorder)=\baseset \times \baseset$ iff $\partialorder$ is a {linear order}. If  $\partialorder \subseteq \baseset \times \baseset$ is a partial order, then $<_\partialorder$ denotes the corresponding strict order, which is irreflexive and transitive.
Linear orders over~$\baseset$ can be given by bijections $\pi: [|\baseset|]\to \baseset$. Hence, $<_\pi$ (or $\leq_\pi$) is used to denote the corresponding strict (or partial) linear order.
Given a binary relation $\alpha$, we denote by $\transitiveclosure(\alpha)$
the transitive closure of $\alpha$.

\begin{definition}[Cocomparability graph]
	Given a partial order $\partialorder \subseteq \baseset \times \baseset$, we let
	$\cocomparabilitygraph{\partialorder} \defeq (\baseset, \baseset\times \baseset \setminus \symmetricclosure(\partialorder))$
	be the \emph{cocomparability graph} of $\partialorder$.
\end{definition}

Given an undirected graph $\graph = (\vertexset, \edgeset)$ and a vertex $\vvertex \in \vertexset$,
we let $\neighbours(\vvertex) \defeq \{\uvertex \mid \uvertex \in \vertexset, (\vvertex, \uvertex) \in \edgeset\}$ be the neighborhood of $\vvertex$.
A \emph{path decomposition} of a graph $\graph = (\vertexset,\edgeset)$ is a sequence $\pathdec = (\bag_1, \bag_2, \ldots, \bag_{\lengthdec})$
of subsets of $\vertexset$, such that the following conditions are satisfied.
\begin{itemize}
	\item $\bigcup_{1 \leq i \leq \lengthdec} \bag_i = \vertexset$.
	\item For each edge $(\uvertex,\vvertex) \in \edgeset$, there is an $i \in [\lengthdec]$ such that
	      $\uvertex,\vvertex \in \bag_i$.
	\item For each $i, j, k\in [\lengthdec]$ with $i < j < k$, $\bag_i \bigcap \bag_k \subseteq B_j$.
\end{itemize}

For each position $p \in \{2, \ldots, \lengthdec\}$, for each vertex $\vvertex \in \bag_{p}\sm \bag_{p-1}$,
we say that $\bag_{p}$ \emph{introduces}~$\vvertex$ ($\vvertex$ is introduced by~$\bag_p$) and for each vertex $\vvertex \in \bag_{p-1}\sm \bag_{p}$,
we say that $\bag_{p}$ \emph{forgets}~$\vvertex$ ($\vvertex$ is forgotten by~$\bag_p$). For a position $p \in \{1, \ldots, \lengthdec\}$, we write $\introduced(p)$
(resp.~$\forgotten(p)$) for the set of all vertices introduced (resp.~forgotten) by $\bag_p$,
and we let
$\leftset_p = \bigcup_{1 \leq i \leq p} \forgotten(p)$
be the set of vertices that have been forgotten up to position $p$.
The \emph{width} of~$\pathdec$ is defined as
$\width(\pathdec) = \max_{i \in [\lengthdec]} |\bag_i| - 1$. The \emph{pathwidth},
$\pw(\graph)$, of $\graph$ is the minimum width of a path decomposition of~$\graph$.

The pathwidth of the cocomparability graph of a partial order may be regarded as a
measure of how close the order is from being a linear order.
The cocomparability graph of a linear order~$\alinearorder$ on $n$ elements is the 
graph with $n$ vertices and no edges. This graph has pathwidth $0$. On the other hand,
if $\alinearorder$ is a partial order where all
$n$ elements are unrelated, then the cocomparability graph of $\alinearorder$ is the
$n$-clique, which has pathwidth $n-1$
(the highest possible pathwidth in an $n$-vertex graph).

%% file: 03-KRA.tex
\section{The Kemeny Rank Aggregation Problem}%
\label{section:KRA}

Let $\candidates$ be a finite set, which in this paper 
will denote a set of candidates, or alternatives. A \emph{partial vote}\footnote{The literature is not clear about these notions. In~\cite{DBLP:journals/tcs/HemaspaandraSV05}, Hemaspaandra et al.~call partial votes that allow ties in linear orders \emph{preference rankings} and explain that this originally goes back to Kemeny. Allowing partial orders (as we do it here) is even more general. The authors of~\cite{DBLP:journals/tcs/HemaspaandraSV05} also consider a different Kendall-Tau distance for preference rankings compared to our setting.}
over $\candidates$ is a partial order over $\candidates$. 
 The KT-distance between two partial votes $\vote_1$ and~$\vote_2$, denoted by $\operatorname{KT-dist}(\vote_1,\vote_2)$, 
is the number of pairs of candidates that are ordered
differently in the two partial votes.
\[\operatorname{KT-dist}(\vote_1,\vote_2)=|\{(\candidate,\candidate')\in \candidates\times \candidates\mid \candidate<_{\vote_1}\candidate'\land \candidate'<_{\vote_2}\candidate\}|\,.\]

Observe that when the votes are totally ordered,
the Kendall-Tau distance can be seen as the `bubble sort' distance,
i.e., the number of pairwise adjacent transpositions needed to transform  one
linear order into the other.
Given a linear order $\vote$ over a set of candidates $\candidates$ and 
a set $\votes$ of votes over $\candidates$, the 
\emph{Kemeny score} of $\vote$ with respect to $\votes$ is defined as 
the sum of the Kendall-Tau distances between $\vote$ and each vote in~$\votes$.
In this work, we consider the following problem.

\problem{\pbGKRAname\ (KRA)}{
A list of partial votes $\votes$ over a set of candidates $\candidates$, 
a non-negative integer $k$.}{Is there a linear order $\vote$ on $\candidates$ such that the
sum of the  KT-distances of~$\vote$ from all the partial votes is  ${}\leq k$?
}
Hence, given partial votes $\pi_1,\dots,\pi_m$ of $\candidates$ and a non-negative integer $k$, the question is if there exists a linear order $\pi\subseteq \candidates\times \candidates$ such that
\(\sum_{i=1}^m \operatorname{KT-dist}(\pi,\pi_i)\leq k\,\).

\begin{definition}
    Given a set $\votes$ of partial votes, the \emph{unanimity order} of $\votes$
    is simply the partial order~$\unanimityorder$ obtained
    as the intersection of all partial orders in $\votes$. In other words, a candidate
    $\candidate_1$ has higher precedence than a candidate $\candidate_2$ in $\unanimityorder$
    if and only if $\candidate_1$ precedes $\candidate_2$ in each vote in $\votes$.
\end{definition}
As a consequence, the more disagreements there are among the voters with respect to the
relative orders of pairs of candidates, the denser the cocomparability graph of
$\unanimityorder$ will be and therefore the greater its pathwidth will be. Therefore,
the pathwidth of the cocomparability graph of the unanimity order of $\votes$ may be
seen as a quantification of the amount of disagreement among the votes in $\votes$.

\paragraph{A notion of diversity for KRA.}
The notion of diversity of solutions for computationally hard 
problems has been considered under a variety of frameworks. In this 
work, we define a notion of diversity for the \textsc{Kemeny Rank Aggregation} 
problem which is analogous to the notion of diversity of vertex sets 
used in \cite{BasteFJMOPR20}. More precisely, if $\rankingset$ is a
set of partial orders, then we define the Kendall-Tau diversity of $\rankingset$ 
as the sum of Kendall-Tau distances between votes in the set $\rankingset$. 

\[\ktdiversity(\rankingset) = \sum_{\vote_1,\vote_2\in \rankingset} \operatorname{KT-dist}(\vote_1,\vote_2)\]

We note that the restricted version of the \textsc{KRA} problem where all votes are 
linear orders, the requirements that a set of solutions is at the same time 
diverse and only contains rankings with small Kemeny score are clashing. 
The problem is that the very existence of two distinct rankings with small
Kemeny score is an impossible task. If two candidates $\candidate_1$ and $\candidate_2$
occur with the order $(\candidate_1,\candidate_2)$ in one of the solutions and 
in the order $(\candidate_2,\candidate_1)$ in the other solution, then at least 
one of these solutions will have a Kemeny score of at least half the number of votes. 
However this opposition between diversity and small Kemeny score is not present in
the setting where votes are allowed to be partial. The generalization to partial votes is
one possible way to circumvent this conflict of desiderata. Another way we will be looking at is
not to consider the cost of the solutions directly but the difference between the cost of solutions
and the cost of an optimal solution. In this case, we can have diversity and a small
difference between the cost and the cost of an optimal solution.

\problem{\pbDivGKRAname\ (Diverse-KRA)}{%
A list of partial votes $\votes$ over a set of candidates $\candidates$,  and 
 $k$, $\numbersolutions$, $\diversityparameter\in\mathbb{N}$.}{Is there 
a set $\rankingset = \{\vote_1,\dots,\vote_{\numbersolutions}\}$ of linear orders
on $\candidates$ such that the Kemeny score for each order $\vote_i$ is at most $k$
and $\ktdiversity(\rankingset)\geq \diversityparameter$?}

\paragraph{Parameterizations of KRA}
The problem \KRAname\ is known to be \NP-complete~\cite{BarTovTri89}, even if only four 
votes are given at the input~\cite{DBLP:conf/www/DworkKNS01}.\footnote{\label{dwork}The proof of this fact is not contained in the conference paper~\cite{DBLP:conf/www/DworkKNS01} but only appears in Appendix B of \url{http://www.wisdom.weizmann.ac.il/~naor/PAPERS/rank_www10.html}.} 
For this reason, \textsc{KRA} has been studied from the perspective of parameterized complexity 
theory under a variety of parameterizations. Below, we consider two prominent parameterizations
for this problem. 

The first parameter we consider is the cost of a solution. 
Simjour~\cite{DBLP:conf/iwpec/Simjour09} obtained an algorithm for the problem that
runs in time $\Oh^{*}(1.403^k)$\longversion{, where $k$ upper-bounds the sum
of the KT-distances of the solution~$\pi$ from all the votes}.
There are also sub-exponential algorithms for  \KRAname\ under this parameterization:
Karpinski and Schudy~\cite{DBLP:conf/isaac/KarpinskiS10} obtained an algorithm for
\KRAname\ that runs in $\Oh^{*}(2^{O(\sqrt{k})})$ time, while
the algorithm of Fernau \emph{et al.}~\cite{DBLP:conf/iwoca/FernauFLMPS10,Feretal2014}, based on a different methodology,
runs in $\Oh^{*}(k^{O(\sqrt{k})})$ time.
Recently, in \cite{AFOW2020}, it was shown that
\textsc{KRA} on instances with only $m=4$  votes on some  candidate set~$C$ and some integer $k$ bounding the sum of the Kendall-Tau
distances to a solution cannot be solved neither in time $\Oh^*\left(2^{o(|C|)}\right)$ nor in time $\Oh^*\left(2^{o(\sqrt{k})}\right)$, unless ETH fails.
The mentioned \NP-hardness of \textsc{KRA} immediately translates to \NP-hardness results of \textsc{KRA} and of \textsc{Diverse-KRA}, in the latter case by setting $r=1$ and $d=0$.

The second parameter we consider is the \emph{unanimity width} of the set of 
votes, which is based on the notion of unanimity order of a set of 
votes~\cite{DBLP:journals/4or/CharonH07}.
The \emph{unanimity width} of $\votes$ is defined as the pathwidth of 
the cocomparability graph of $\unanimityorder$.

\subsection{Some Discussion on \textsc{(G)KRA}}

In this subsection we briefly describe some toy applications where the notion
of diversity can be naturally combined with the notion of Kemeny Rank Aggregation, both 
in the totally ordered setting and in the partially ordered setting.

\paragraph{A Natural Class of Partial Orders of Low Width.}
A {\em $k$-bucket order}~\cite{DBLP:conf/pods/FaginKSMV04} is a partial order $\partialorder$ where the set of vertices 
can be partitioned into a sequence of clusters $\acluster_1,\dots,\acluster_{m}$, each of size at most $k$, where
for each $i\in [m-1]$, all elements in $\acluster_i$ precede all elements in $\acluster_{i+1}$.
In the context of democratic scheduling, a unanimity order that is a $k$-cluster order corresponds to the situation where 
the tasks to be executed are split into work-packages (the clusters), each containing at most $k$ tasks,
where the order of execution of the work-packages is agreed on, 
but the order inside each cluster is not. 
It is easy to see that the cocomparability graph of a $k$-cluster order has pathwidth at most $k$.
Additionally, for each $m$ and $k$, there are sets of votes whose unanimity order $\unanimityorder$ is a $k$-cluster order
such that for each diversity $\diversity$, there is a set $S$ containing $\diversity$ linear extensions of $\unanimityorder$, 
each of which has optimal Kemeny score, and such that $S$ has maximum diversity.  

\paragraph{A Concrete Application of Bucket Orderings.}
Consider an election with 5 candidates $A,B,C,D,E$ for 3 positions and 100 voters.
If voters are forced to put (strict) linear orders, then it could be that there might be 50 votes like $A<B<C<D<E$ and 50 votes like $A<B<D<C<E$.
There are two optimum Kemeny solutions, each of them coinciding with the two types of votes that were cast.
But even if these two solutions are put into a diverse set of solutions, then the distance between these to votes is one. However,  the sum of the Kemeny-Tau distances of any of the  two optimal solutions to all given votes is  50. Hence, although the votes do agree to quite some extent, the resulting numbers are relatively big.
But are these strict linear orderings really expressing the opinions of the voters?
This has been discussed in the social choice literature, and there is some evidence that many people do not have a strict preference among \emph{all} candidates, but ranking them in groups is more realistic.
This is our main motivation to introduce the \textsc{KRA} model.

For instance, it could be that 10 voters do not care about the ranking of $A$ versus $B$, but they would rank them all above $C,D,E$, without caring too much about their sequence, either.
Another 10 voters might not care about the exact sequence of $A$ and $B$, nor about the sequence of $D$ and $E$, but they clearly put $A$ and $B$ before $D$ which is in turn ahead of $C$ and~$E$.
In shorter notation, we get $A=B<D<C=E$.
There might be more different votes, as altogether summarized in the 
following table: \\

\noindent
\centerline{\fbox{\begin{tabular}{lll}
type I &10 voters & $A=B<C=D=E$\\
type II & 10 voters &$A=B<D<C=E$\\
type III & 10 voters & $A=B=C<D=E$\\
type IV & 40 voters & $A=B=C=D<E$\\
type V & 20 voters & $A<B<C=D=E$
\end{tabular}}} 
\\
\\ 

This election could be turned into the first example if all voters would have been forced to commit themselves to linear orderings. Obviously, the 50 type II and type IV voters are compatible with $A<B<D<C<E$, while all but the typ II voters are compatible with $A<B<C<D<E$.

Hence, when viewing this as an instance of \textsc{KRA}, the ranking $A<B<C<D<E$ would clearly win, as its Kemeny score is just 10.
(Hemaspaandra et al.~\cite{DBLP:journals/tcs/HemaspaandraSV05} would attribute a much higher value here.)
However, the diversity between $A<B<D<C<E$ and $A<B<C<D<E$ stays one.
It might be also possible to consider the solution $B<A<C<D<E$ now, or even
$A<B<C<E<D$. While in the model where we required all votes to be linear orderings, only the mentioned two solutions would make sense in a diverse solution that should not be too expensive in terms of costs, here it might be possible to look for three or four solutions in the diverse set.
This is another aspect that makes our model interesting for election problems.

\paragraph{Diversity in Budget Allocation.}
Suppose a decision-maker wants to allocate funds for the implementation of 
a set $\{\aproject_1,\dots,\aproject_{n}\}$ of projects, which would be executed sequentially. In order
to determine the priority in which the funds will be allocated, the decision-maker asks each voter 
to rank the projects in the desired order of execution.  The goal of the decision-maker
is to come up with an order that maximizes satisfaction among the voters. 
Nevertheless, instead of having a unique final ranking computed completely automatically, the decision-maker
may prefer to have in hands a diverse set of sufficiently good rankings.
This would allow the decision-maker to also take into consideration important external factors before taking 
a final decision for the allocation of funding. Such external factors may be budgetary constraints, environmental
constraints, compliance with regulations, etc. 

\paragraph{Diversity in Search Rankings.} 
Search engines are part of our everyday live. But who is really following the links presented by the search engine beyond the first few pages that are displayed on the user's screen?
Therefore, it is crucial that important and interesting information is put on the very first pages. Usually, search engines consider some sort of relevance measure to rank the answers.
Clearly, the search engine knows a bit more. For instance, is it important to display different hits from the same domain? Rather, it would be better for the user to see ``really different'' hits on the first page.
Our concept of diversity could be implemented on two levels here: Either, we build a meta-search engine that collects the rankings of answers from different search engines (on the same question) and tries to come up with a diverse set of rankings that could help build the first couple of pages. Or, we consider different ranking functions within a search engine itself; 
in this second scenario, there could be also ties, so that our more general framework would apply.

\paragraph{Diversity in Team Formation.} 
An organization wants to form a team/committee to perform some task and there 
are several candidates. But the committee will have only a few members (say three). 
To choose the committee, the organization will pick a rank with sufficiently
good score and select the first three candidates of that particular rank.
The intuition is that candidates that appear in the first three positions of 
some rank with good score have enough legitimacy to take the role. 
On the other hand, it may be important for the organization to have some liberty
to choose which rank they will be using due to external factors, such as political, 
social, or affinity factors. For instance, if in one of the sufficiently 
good rankings the first three candidates are male, and in another sufficiently  good  
ranking we have two female and one male, then it may be better to 
pick the latter one for gender equality reasons.

\paragraph{Diversity in Planning of Menus.}
Suppose a pizza delivery service service wants to optimize their menu. Each of the current
customers is asked to rank his/her four preferred toppings. All other toppings are equally
ranked below the fourth one. The goal is to obtain a small, sufficiently diverse, set of
rankings (say, containing 10 rankings). Each of these rankings will correspond to a 
pizza on the menu. In order to cover many customer tastes, high diversity would be very
welcome, so that each customer could find her favorite among the (only) 10 pizzas on the menu.

%% file: 04-CO.tex
\section{Completion of an Ordering}
\label{section:CO}
In this section, we will introduce the \COname{} problem, a generalization of the
\PCOname{} (PCO) problem originally considered in~\cite[Sec.~8]{DBLP:conf/gd/DujmovicFK03}
and~\cite[Sec.~6.4]{Fer05a}.

\problem{\pbCOname{} (CO)}{
A partial order $\partialorder \subseteq \baseset \times \baseset$ over a set $\baseset$,
a cost function $\cost:\baseset \times \baseset \to \N$, and some 
 $\parameter \in \N$.
}
{Is there a linear order $\linearorder \supseteq \partialorder$ with $\cost(\linearorder\setminus \partialorder)= \sum_{(x,y)\in\tau\setminus\rho}\cost(x,y)\leq \parameter$?}

Intuitively, given a partial order $\partialorder$ and a cost function $\cost$, the goal 
is to find a linear extension of~$\partialorder$ incurring a cost of at most~$\parameter$. The 
only difference between \textsc{CO} and the original \textsc{PCO} problem introduced 
in~\cite{DBLP:conf/gd/DujmovicFK03,Fer05a} is that, 
in the latter, the cost function needs to satisfy the following condition: 
for every pair $(x,y) \in \baseset \times \baseset$ such that $x$ and $y$ are incomparable
in $\partialorder$, the cost of $(x,y)$ is strictly positive ($\cost(x,y) > 0$).
In \textsc{CO}, for such pair the cost can be zero ($\cost(x,y) > 0$).

\paragraph{A Diversity Measure for \textsc{CO}.}
We note that the notion of Kendall-Tau diversity introduced in Section 
\ref{section:KRA} can also be used as a notion of diversity for \textsc{CO}, i.e., 
given a set $\rankingset$ of (not necessarily optimal) solutions for a
given instance $(\apartialorder,\cost)$ of \textsc{CO}, we let $\ktdiversity(\rankingset)$
be the diversity of this set.

\problem{\pbDivCOname{} (Diverse-CO)}{
A partial order $\partialorder \subseteq \baseset \times \baseset$ over a set $\baseset$,
a cost function $\cost:\baseset \times \baseset \to \N$, and non-negative integers
$\parameter,\numbersolutions,\diversityparameter\in \N$.}
{Is there a set $\rankingset = \{\linearorder_1,\dots,\linearorder_{\numbersolutions}\}$ of linear
extensions of $\partialorder$ such that
$\cost(\linearorder_i\setminus \partialorder)\leq \parameter$ for each $i\in [\numbersolutions]$, and 
$\ktdiversity(\rankingset)\geq \diversityparameter$ 
?}

\paragraph{Reducing KRA to CO.}
Next, we give a rather straightforward reduction from 
\textsc{KRA} to \textsc{CO}. 
Given an instance $(\votes, \candidates)$ of \textsc{KRA}
with partial votes $\votes = (\vote_1, \ldots, \vote_\mvotes)$  and candidates $\candidates=\{\candidate_1,\ldots,\candidate_\ncandidates\}$, 
we construct an instance $(\apartialorder,\cost)$ of \textsc{CO}
by letting $\apartialorder$ be the unanimity order of $\votes$, and 
by defining the cost function $\cost:\candidates\times \candidates \rightarrow \N$
as follows.
For every pair of candidates  $(\candidate, \candidate')$, 
we define its \emph{cost}, 
$\cost(\candidate, \candidate')$, as the number of votes that order
$\candidate'$ before~$\candidate$. More formally,
$\cost(\candidate, \candidate') = |\{i \in [\mvotes] \mid \candidate' <_{\vote_i} \candidate\}|$.
With this reduction, it is straightforward to check that a given 
linear order $\vote$ of the candidates has Kemeny score

\begin{equation*}
\begin{split}
\sum_{i=1}^m \operatorname{KT-dist}(\pi,\pi_i)=&\sum_{i=1}^m \sum_{j=1}^n\sum_{k=1}^n [\candidate_j <_{\vote_i} \candidate_k \land \candidate_k <_{\vote} \candidate_j ]\\
=&\sum_{j=1}^n\sum_{k=1}^n\cost(\candidate_k, \candidate_j)[\candidate_k <_{\vote} \candidate_j ].
\end{split}
\end{equation*}

Here, for a logical proposition $p$, we use the bracket notation [$p$] 
to denote the integer $1$ if $p$ is true the integer $0$ if $p$ is false. 
In other words, $\sum_{i=1}^m \operatorname{KT-dist}(\pi,\pi_i)$ is the
cost of~$\pi$ as a linear extension of the ordering $\apartialorder$. 

It is important to note that if all votes in $\votes$ are linear orders then $(\apartialorder,\cost)$
is actually an instance of \textsc{PCO}. In other words, if two candidates $\candidate$ and 
$\candidate'$ are incomparable in the unanimity order, then the cost assigned by $\cost$ 
to both pairs $(\candidate,\candidate')$ and $(\candidate',\candidate)$ are strictly 
positive. This property will be used crucially in the development of our sub-exponential
time algorithm for KRA parameterized by cost.  

We also note that since our reduction is solution preserving, it is also 
immediate that it is diversity preserving. In other words, $\rankingset$ is 
a set of solutions of diversity $d$ for an instance of KRA if and only if 
it is also a set of solutions of diversity $d$ for the corresponding instance of CO.

%% file: 05-DPPCO.tex
\section{Diverse \textsc{CO} Parameterized by Pathwidth}%
\label{section:DPCO}

In this section, we devise a fixed parameter tractable algorithm for
\textsc{Diverse CO} parameterized by solution imperfection, number of solutions, scatteredness, and
pathwidth of the cocomparability graph of the input instance.
Given our reduction that preserves solution and parameters from \textsc{KRA} to \textsc{CO} introduced
in Section~\ref{section:CO}, this algorithm immediately implies that
\textsc{Diverse KRA} is fixed parameter tractable when parameterized by
solution imperfection, number of solutions, scatteredness, and unanimity width.

We start by defining a suitable notion of consistency between
a path decomposition and a given partial order.
Let $\agraph = (\baseset,\edgeset)$ be a graph
and $\apartialorder \subseteq \baseset \times \baseset$ be a partial order
on the vertices of $\agraph$. We say that  a path decomposition
$\pathdecomposition = (\bag_1,\ldots, \bag_{\lengthdecomposition})$
is \emph{$\apartialorder$-consistent} if there is no pair of vertices
$(\xelement, \yelement) \in \apartialorder$ such that
\[\max(\{i \in [\lengthdecomposition] \mid \yelement \in \bag_i\}) < \min(\{i \in [\lengthdecomposition] \mid \xelement \in \bag_i\}).\]
In other words, if $\xelement$ is smaller than $\yelement$ in $\apartialorder$, then $\yelement$
cannot be forgotten in	$\pathdecomposition$ before $\xelement$ is introduced in $\pathdecomposition$.
The \emph{$\apartialorder$-consistent pathwidth} of $\agraph$, denoted by $\cpw(\agraph, \apartialorder)$,
is the minimum width of a $\apartialorder$-consistent path decomposition of $\agraph$.

It has been shown recently that for any partial order $\apartialorder\subseteq \baseset\times \baseset$,
the pathwidth of the cocomparability graph $\agraph_{\apartialorder}$ is equal to the consistent pathwidth
of $\graph_{\apartialorder}$ \cite{AFOW2020}. The proof of this result was based on the fact
that the consistent pathwidth
of a cocomparability graph of a partial order is equal to the interval width of the order \cite{habib1994treewidth}.
Nevertheless, the problem of constructing, or even approximating, a minimum-width consistent path decomposition
in FPT time was left open in \cite{AFOW2020}.

By taking a closer look at the theory of cocomparability graphs, we solve this open problem
in a constructive way. More precisely, in Lemma~\ref{lemma:ConstructionDecomposition} we show
that for any partial order $\apartialorder$,
one can construct a $\apartialorder$-consistent path decomposition of the cocomparability graph
$\agraph_{\apartialorder}$ in fixed-parameter tractable time parameterized by the pathwidth of the graph~$\agraph_{\apartialorder}$.

\begin{lemma}%
	\label{lemma:ConstructionDecomposition}
	Let $\apartialorder\subseteq \baseset\times \baseset$ be a partial order and
	$\agraph_{\apartialorder}$ be the cocomparability graph of
	$\apartialorder$. Then one can construct a nice $\apartialorder$-consistent path decomposition $\pathdec$
	of $\agraph_{\apartialorder}$ of width $\Oh(\pw(\agraph_\apartialorder))$ in time $2^{\Oh(\pw(\agraph_\apartialorder))}\cdot |\baseset|$.
\end{lemma}
\begin{proof}
	Let $\partialorder$ be a partial order over a set $\baseset$ and $\graph_\partialorder$
	be the cocomparability graph of $\partialorder$. It has been shown in Theorem 2.1 of
	\cite{habib1994treewidth} that any minimal triangulation $H$ of $\graph_\apartialorder$
	is not only a cocomparability graph, but also an interval graph. This result allows us
	to compute a $\apartialorder$-consistent path decomposition of $\agraph_{\apartialorder}$
	as follows.

	We start by computing a tree decomposition $\tree$ of $\agraph_{\apartialorder}$ of width at most $5\cdot \pw(\agraph)+4$
	in time $2^{\Oh(\pw(\agraph))}\cdot |\baseset|$ using the algorithm from \cite{BodlaenderDDFLP16}. Subsequently we
	construct a triangulation $H_{\tree}$ of $\agraph_{\apartialorder}$ by transforming each bag of the decomposition
	$\tree$ into a clique. More precisely, we add an edge to vertices $u$ and $v$ in $\graph_{\apartialorder}$ if and only
	if $u$ and $v$ occur together in some bag. This operation clearly preserves treewidth, since the size of the bags
	do not increase. Therefore, the graph~$H_{\tree}$ is a triangulation of $\agraph_{\apartialorder}$ of treewidth at most
	$5\cdot \pw(\agraph)+4$. Now, we successively delete edges from $H_{\tree}$ until we get a minimal triangulation $H$ of
	$\agraph_{\apartialorder}$. In other words, by removing any additional edge from $H$, we either get a graph that is not
	triangulated or that is not a supergraph of $\agraph_{\apartialorder}$.
	We have $\edges(\graph_\partialorder) \subseteq \edges(H) \subseteq \edges(H_\tree)$ and $\edges(H)$ is minimal with respect
	to inclusion. Therefore, $\tw(H) \leq \tw(H_\tree) = \width$. By \cite{habib1994treewidth}, we know that $H$ is an interval graph.

	Now, adding an edge  $\{\uvertex,\vvertex\}$ to $\graph_{\apartialorder}$ is equivalent to
	removing the edge constructing the DAG $\apartialorder\backslash \{(\uvertex, \vvertex)\}$. What is shown in~\cite{habib1994treewidth} is that
	the DAG $\apartialorder \backslash \{(\uvertex,\vvertex) \mid \{\uvertex,\vvertex\}\in \edges(H)\backslash \edges(\agraph_{\apartialorder})\}$
	is actually a partial order $\intervalorder$. Note that when deleting edges from a partial order, the only axiom that can be broken is transitivity.
	So what this result is really saying is that by deleting the pairs corresponding to edges that are in $H$ but not in $G$, we can indeed preserve
	transitivity. The crucial fact about this construction is that the partial order $\intervalorder$ is actually an interval order, and therefore
	$H$ is an interval supergraph of $\agraph_{\apartialorder}$.

	Now, from the interval graph $H$, we derive a $\partialorder$-consistent path decomposition $\pathdec$ of $\graph_\partialorder$.
	This construction is as follows. Given two maximal cliques $X$ and $Y$ of $H$, we say that $X$ is smaller than $Y$ if there exist
	vertices $x \in X$ and $y \in Y$ such that $(x,y) \in \intervalorder$. This relation defines a linear order on the maximal
	cliques~\cite{DBLP:journals/tcs/HabibMPV00}. It follows from~\cite{DBLP:journals/tcs/HabibMPV00} that the sequence of maximal cliques
	obtained by ordering the maximal cliques of $H$ according to the order above is a path decomposition of $H$.
	As the bags of the path decomposition follow the order above, this path decomposition is
	consistent with $\intervalorder$.
	Now, since any path decomposition of $H$ is also a path decomposition of $\agraph_\partialorder$, and
	as $\intervalorder \subseteq \partialorder$, this path decomposition is also $\partialorder$-consistent.

	We note that the process of finding all maximal cliques of an interval graph $H$ can be realized in time linear in the size of $H$~\cite{DBLP:journals/tcs/HabibMPV00}. Note that since $H$ has pathwidth $\Oh(\pw(\agraph_{\apartialorder}))$, the number of edges
	of $H$ is bounded by $\Oh(\pw(\agraph_{\apartialorder})^2\cdot |\baseset|)$. So the process of finding the maximal cliques in $H$
	takes time $\Oh(\pw(\agraph_{\apartialorder})^2\cdot |\baseset|)$. Finally, since the most time-expensive part of the process
	described above is the construction of the tree decomposition $\tree$, we have that the whole process takes time
	$2^{\Oh(\pw(\agraph_{\apartialorder}))}\cdot |\baseset|$.
\end{proof}

Let $\apartialorder$ be a partial order and $\pathdec = (\bag_1, \bag_2, \ldots, \bag_{\lengthdecomposition})$
be a $\apartialorder$-consistent path decomposition of~$\agraph_{\apartialorder}$ of width $\width$.
For each $\aposition\in [\lengthdecomposition]$, let $\pairs_{\aposition}$ be  the set of pairs of the form
$(\asubsetbag,\linearorder)$
where $\asubsetbag$ is a subset of $\bag_\aposition$ that contains vertices introduced
by $\bag_\aposition$ ($\introduced(\aposition) \subseteq \asubsetbag\subseteq \bag_{\aposition}$),
$\linearorder\supseteq\apartialorder|_\asubsetbag$ is a
linear extension of the restriction $\apartialorder|_\asubsetbag\defeq \asubsetbag\times \asubsetbag \cap \apartialorder$ of $\apartialorder$ to $\asubsetbag$.

\begin{definition}
	Let $\aposition\in [\lengthdecomposition]$, $\distanceoptimality \in \N$, and $\optimalfunction:\pairs_{\aposition}\rightarrow \N$. Then, we
	let $\triples_{\aposition}(\optimalfunction,\distanceoptimality)$ be the set of all triples of the form $(\asubsetbag,\linearorder,\partialcost)$,
	where $(\asubsetbag,\linearorder)\in \pairs_{\aposition}$ and $\optimalfunction(\asubsetbag,\linearorder)\leq \partialcost\leq \optimalfunction(\asubsetbag,\linearorder)+\distanceoptimality$.
\end{definition}

Intuitively, the function $\optimalfunction$ will be used by our dynamic programming algorithm to record the
optimal values of partial solutions at each bag $\bag_{\aposition}$ when processing the path decomposition
from left to right (see Theorem~\ref{theorem:SolutionConstructionOne} and Theorem~\ref{theorem:SolutionConstructionDiverse}) and $\distanceoptimality$ will be the allowed solution imperfection. In the case
of a unique solution, this value will be $0$ but this parameter will be useful in the
diverse case as we allow sub-optimal linear extensions.
A partial solution up to the $\aposition$-th bag is a linear order $\astrictorder$ of $\bigcup_{j\leq\aposition} \bag_j$.
Vertices that will be introduced in future bags can be inserted
in~$\astrictorder$, to extend it, only after vertices already forgotten in the $\aposition$-th bag.
If $\uvertex$ will be introduced
in a future bag and $\vvertex$ is in some $\bag_j$ but not in~$\bag_\aposition$, then by consistency of the
path decomposition with respect to $\apartialorder$, we have $\vvertex <_\apartialorder \uvertex$.
Therefore, in $\bag_\aposition$,
we only need to remember the ``last'' part of $\astrictorder$, which are the vertices that are in $\bag_\aposition$
and after all forgotten vertices in $\astrictorder$.

\begin{remark}%
	\label{remark:UpperBoundTriples}
	For each $\aposition\in [\lengthdecomposition]$, $\optimalfunction:\pairs_{\aposition}\rightarrow \N$ and $\distanceoptimality\in \N$, the size of $\triples_{\aposition}(\optimalfunction,\distanceoptimality)$
	is bounded by $e\cdot (\distanceoptimality+1)\cdot (\width+1)!$.
\end{remark}
\begin{proof}
	Given a bag $\bag_\aposition$ at position $\aposition$, the size of $\triples_{\aposition}(\optimalfunction,\distanceoptimality)$ is bounded by:
	\[(\distanceoptimality+1) \cdot \sum_{0 \leq i \leq |\bag_\aposition|} \binom{|\bag_\aposition|}{i}\cdot i! \leq e \cdot (\distanceoptimality+1) \cdot |\bag_\aposition|!
		\leq e \cdot (\distanceoptimality+1) \cdot (\width+1)!\]
	where $\binom{|\bag_\aposition|}{i}$ is the number of subsets of $\bag_\aposition$ of
	size $i$ and $i!$ is the number of possible ordering of a set of size $i$.
\end{proof}

For each $\aposition\in [\lengthdecomposition - 1]$, $\optimalfunction:\pairs_{\aposition}\rightarrow \N$
and $\distanceoptimality\in \N$, we say that a triple
$(\asubsetbag,\linearorder,\partialcost)\in \triples_{\aposition}(\optimalfunction,\distanceoptimality)$
is \emph{compatible} with a triple  $(\asubsetbag',\linearorder',\partialcost')\in
	\triples_{\aposition+1}(\optimalfunction,\distanceoptimality)$
if the following conditions are satisfied.

\begin{enumerate}[label=C\arabic*]
	\item\label{prop_1} Let $\vvertex = \max_{\linearorder} (\asubsetbag\cap\forgotten(\aposition+1))$ be the
	      maximum vertex of $\asubsetbag$ forgotten by $\bag_{\aposition+1}$ according to the linear order $\linearorder$.
	      Then, we have the following equality $\asubsetbag' = \introduced(\aposition+1)
		      \cup \{\uvertex \in \asubsetbag \mid \vvertex <_{\linearorder} \uvertex\}$. This
	      means that one can build $\asubsetbag'$ from $\asubsetbag$ by removing 
		all vertices that are either forgotten by $\bag_{\aposition+1}$ or smaller than some vertex forgotten by $\bag_{\aposition+1}$, and 
		subsequently, by adding all vertices that have been introduced by $\bag_{\aposition+1}$. 
	\item\label{prop_2} $\linearorder|_{\asubsetbag\cap \asubsetbag'} = \linearorder'|_{\asubsetbag\cap \asubsetbag'}$,
	      i.e., $\linearorder$ and
	      $\linearorder'$ agree on $\asubsetbag \cap \asubsetbag'$.
	\item\label{prop_3} $\partialcost' = \partialcost + \sum_{\vvertex \in \introduced(\aposition+1)}
		      (\sum_{\uvertex \in \asubsetbag', \uvertex <_{\linearorder'} \vvertex} \cost(\uvertex, \vvertex) +
		      \sum_{\uvertex \in \asubsetbag\cap\asubsetbag', \vvertex <_{\linearorder'} \uvertex} \cost(\vvertex,\uvertex) +
		      \sum_{\uvertex \in \bag_{\aposition+1}\sm\asubsetbag'} \cost(\uvertex, \vvertex))$.
	      To compute $\partialcost'$, we add to
	      $\partialcost$ the cost of adding the introduced vertices in the order. The first
	      two terms compute the cost of each new vertex in $\linearorder'$ and the last one
	      computes the cost of placing the new vertices after all vertices in $\bag_{\aposition+1}\sm\asubsetbag'$.
\end{enumerate}

A \emph{compatible sequence} for $\pathdec$ is a sequence of triples
$\compatiblesequence = (\asubsetbag_1,\linearorder_1,\partialcost_1) \dots
	(\asubsetbag_{\lengthdecomposition},\linearorder_{\lengthdecomposition},\partialcost_{\lengthdecomposition})$
such that for each $\aposition \in [\lengthdecomposition]$,
$(\asubsetbag_{\aposition},\linearorder_{\aposition},\partialcost_{\aposition})$
is compatible with
$(\asubsetbag_{\aposition-1},\linearorder_{\aposition-1},\partialcost_{\aposition-1})$.

Our interest in compatible sequences stems from the two following lemmas.

\begin{lemma}
	\label{lemma:CompatibleSequenceOne}
	Let $\apartialorder\subseteq \baseset\times\baseset$ be a partial order over $\baseset$,
	$\cost:\baseset\times \baseset\to \naturalIntegers$ be a cost function,
	and $\pathdec$ be a $\apartialorder$-consistent path decomposition of the graph $\agraph_{\apartialorder}$.
	Let
	$$\compatiblesequence = \frak{t}_1\dots \frak{t}_{\lengthdecomposition} = (\asubsetbag_1,\linearorder_1,\partialcost_1) \dots
		(\asubsetbag_{\lengthdecomposition},\linearorder_{\lengthdecomposition},\partialcost_{\lengthdecomposition})$$
	be a compatible sequence for $\pathdec$. Then, the linear order
	$\vote = \transitiveclosure(\apartialorder \cup \linearorder_1 \cup \dots \cup \linearorder_{\lengthdecomposition})$ is
	a linear extension of $\apartialorder$ of cost $\partialcost_{\lengthdecomposition}$.
\end{lemma}
\begin{proof}
	Lemma \ref{lemma:CompatibleSequenceOne} follows straightforwardly by the following claim,
	which can be proved by induction on $\aposition$.
	\begin{claim}
		For each position $\aposition \in \intv{\lengthdecomposition}$,
		$\vote_\aposition = \transitiveclosure(\apartialorder|_{\leftset_\aposition\cup\bag_\aposition} \cup \linearorder_1\cup \dots \cup \linearorder_{\aposition})$
		is a linear extension of $\apartialorder|_{\leftset_\aposition\cup\bag_\aposition}$
		of cost $\partialcost_\aposition$.
	\end{claim}
	In the base case, $\asubsetbag_1 = \bag_1 = \introduced(1)$ and by definition $\linearorder_1 =
		\vote_1$ is a linear extension
	of $\partialorder|_{\bag_1}$ and $\partialcost_1 = \cost(\linearorder_1)$.

	Now, let $\aposition \in \intv{\lengthdecomposition}$ be a position in $\pathdec$,
	and suppose that the claim holds for $\aposition$. We show that it also holds for $\aposition+1$.
	For that, we need to check that the transitive closure of $\apartialorder|_{\leftset_\aposition\cup\bag_\aposition}
		\cup \linearorder_1\cup \dots \cup \linearorder_{\aposition+1}$ defines a linear extension of
	$\apartialorder|_{\leftset_{\aposition+1} \cup \bag_{\aposition + 1}}$. This means that
	$\vote_{\aposition + 1}$ does not contain loops and each pair $\uvertex, \vvertex \in \leftset_{\aposition+1} \cup
		\bag_{\aposition + 1}$ is ordered by $\vote_{\aposition + 1}$.
	By \ref{prop_2}, we have that $\transitiveclosure(\apartialorder \cup \linearorder_1\cup \dots \cup \linearorder_{\aposition})$
	is compatible with $\linearorder_{\aposition+1}$ and $\apartialorder$ and so there is no
	loop in $\vote_{\aposition + 1}$.
	Let $\uvertex, \vvertex \in \leftset_{\aposition + 1} \cup \bag_{\aposition + 1}$.
	If $\uvertex, \vvertex \in \asubsetbag_{\aposition + 1}$, then $\uvertex$ and $\vvertex$
	are ordered by $\linearorder_{\aposition + 1}$ and thus by $\vote_{\aposition + 1}$.
	If $\uvertex, \vvertex \notin \asubsetbag_{\aposition + 1}$, then $\uvertex, \vvertex \in
		\leftset_{\aposition} \cup \bag_{\aposition}$, therefore they are ordered by $\vote_{\aposition}$
	and thus by $\vote_{\aposition + 1}$.
	If $\uvertex \in \asubsetbag_{\aposition + 1}$ and $\vvertex \notin \asubsetbag_{\aposition + 1}$,
	then, by \ref{prop_1}, we have $\vvertex <_{\linearorder_\aposition}
		\max_{\linearorder_\aposition}(\asubsetbag_\aposition \cup \forgotten(\aposition + 1))$
	and by $\partialorder$-consistency of $\pathdec$, we have $\uvertex >_{\apartialorder}
		\max_{\linearorder_\aposition}(\asubsetbag_\aposition \cup \forgotten(\aposition + 1))$.
	Therefore, by transitivity we have $\vvertex <_{\vote_{\aposition + 1}} \uvertex$.
	By \ref{prop_3}, we have
	$\partialcost_{\aposition+1} = \cost(\vote_{\aposition+1})$.
\end{proof}

\begin{lemma}%
	\label{lemma:CompatibleSequenceTwo}
	Let $\apartialorder\subseteq \baseset\times\baseset$ be a partial order over $\baseset$,
	$\cost:\baseset\times \baseset\to \naturalIntegers$ be a cost function,
	and $\pathdec$ be a $\apartialorder$-consistent path decomposition of the graph $\agraph_{\apartialorder}$.
	Let $\vote$ be a linear extension of $\apartialorder$, and
	$\compatiblesequence = (\asubsetbag_1,\linearorder_1,\partialcost_1) \dots
		(\asubsetbag_{\lengthdecomposition},\linearorder_{\lengthdecomposition},\partialcost_{\lengthdecomposition})$
	be a sequence such that for each position $\aposition \in \intv{\lengthdecomposition}$,
	$\asubsetbag_\aposition = \{\vvertex \in \bag_\aposition \mid \vvertex >_\vote \max_\vote(\leftset_\aposition)\}$,
	$\linearorder_\aposition = \vote|_{\asubsetbag_\aposition}$, and
	$\partialcost_\aposition = \cost(\vote|_{\leftset_\aposition \cup \bag_\aposition})$.
	Then, $\compatiblesequence$ is a compatible sequence for $\pathdec$.
\end{lemma}
\begin{proof}
	One can see that this construction satisfies conditions \ref{prop_1}, \ref{prop_2} and \ref{prop_3}
	for each position in the path decomposition.
\end{proof}

Lemma~\ref{lemma:CompatibleSequenceOne} and Lemma~\ref{lemma:CompatibleSequenceTwo} immediately yield an \FPT\ dynamic programming algorithm for computing
a linear extension of $\apartialorder$.
To define the algorithm more precisely, we first need to define the set of functions $\optimalfunction_\aposition$
that we will use to define a set of triples with $\triples_\aposition$. For each
$\aposition \in \intv{\lengthdecomposition}$, we define $\optimalfunction_\aposition: \pairs_\aposition \to \N$
as follows. For each $(\asubsetbag, \linearorder)\in \pairs_\aposition$, we let $\partialcost$
be the minimum cost of a partial solution~$\vote$ up to bag $\aposition$ such that
$\asubsetbag = \{\vvertex \in \bag_\aposition \mid \vvertex >_\vote \max_\vote(\leftset_\aposition)\}$ and
$\linearorder = \vote|_{\asubsetbag_\aposition}$; then we let
$\optimalfunction_\aposition(\asubsetbag, \linearorder) = \partialcost$. Intuitively,
$\optimalfunction_\aposition$ associate to each linear order $\linearorder$ the cost of
an optimal partial solution ``ending'' by $\linearorder$.
Now, we will describe the algorithm, we process the path decomposition from left to right
in $\lengthdecomposition$ time steps, where at each time step $\aposition$, we construct
the value of $\optimalfunction_\aposition$ that we need and a subset
$\promisingtriples_{\aposition}\subseteq \triples_{\aposition}(\optimalfunction_\aposition, 0)$ of \emph{promising triples}, which are, intuitively,
triples that have a potential to lead to an optimal solution. At time step $1$, we let
$\promisingtriples_1 = {\{(\bag_1,\linearorder,\cost(\linearorder))\}}_{\linearorder \text{ is a linear extension of } \partialorder|_{\bag_1}}$.
At each time step $\aposition\geq 2$, $\promisingtriples_{\aposition}$
is the set of all triples in $\triples_{\aposition}(\optimalfunction_\aposition, 0)$ that are compatible with some triple in $\promisingtriples_{\aposition-1}$.
At the end of the process, assuming  that $\promisingtriples_{\aposition}$ is non-empty for each $\aposition\in [\lengthdecomposition]$,
we can reconstruct a compatible sequence by backtracking. First, by selecting an arbitrary triple
$\frak{t}_{\lengthdecomposition}$ in $\promisingtriples_{\lengthdecomposition}$, then by selecting an
arbitrary triple $\frak{t}_{\lengthdecomposition-1}$ in $\promisingtriples_{\lengthdecomposition-1}$ compatible
with $\frak{t}_{\lengthdecomposition}$, and so on. Once we have constructed a compatible sequence
$\frak{t}_1\dots \frak{t}_{\lengthdecomposition}$, we can extract a linear extension $\vote$ of cost~$\partialcost_{\lengthdecomposition}$
by setting $\vote = \transitiveclosure(\apartialorder \cup \linearorder_1\cup \dots \linearorder_{\lengthdecomposition})$.
This description gives rise to the following theorem.

\begin{theorem}
	\label{theorem:SolutionConstructionOne}
	Let $\apartialorder\subseteq \baseset\times \baseset$, let $n = |\baseset|$, let $\width$ be the pathwidth
	of the cocomparability graph of $\apartialorder$,
	and $\cost:\baseset\times \baseset\rightarrow \intv{m}_0$
	be a cost function.
	Then, one can compute an optimal solution in time $\Oh\left({\width!}^{\Oh(1)} \cdot n
		\cdot \log (n \cdot m)\right)$.
\end{theorem}
\begin{proof}
	By Lemma~\ref{lemma:ConstructionDecomposition},
	one can construct a nice $\apartialorder$-consistent path decomposition $\pathdec$
	of $\agraph_{\apartialorder}$ of width $\Oh(\width)$ in time $2^{\Oh(\width)} \cdot n$.

	Lemma~\ref{lemma:CompatibleSequenceTwo} shows that, if a solution $\vote$ exists, then
	there exist a compatible sequence associated to it. Now we will show that
	this sequence is actually build by the algorithm.
	Let $\compatiblesequence$ be the sequence define in Lemma~\ref{lemma:CompatibleSequenceTwo}.
	We recall that the sequence is define as follow
	$\compatiblesequence = (\asubsetbag_1,\linearorder_1,\partialcost_1) \dots
		(\asubsetbag_{\lengthdecomposition},\linearorder_{\lengthdecomposition},\partialcost_{\lengthdecomposition})$,
	where for each position $\aposition \in \intv{\lengthdecomposition}$,
	$\asubsetbag_\aposition = \{\vvertex \in \bag_\aposition \mid \vvertex >_\vote \max_\vote(\leftset_\aposition)\}$,
	$\linearorder_\aposition = \vote|_{\asubsetbag_\aposition}$ and
	$\partialcost_\aposition = \cost(\vote|_{\leftset_\aposition \cup \bag_\aposition})$.
	First, because $\vote$ is optimal, one can easily check that for each $\aposition\in\intv{\lengthdecomposition}$,
	$(\asubsetbag_\aposition,\linearorder_\aposition,\partialcost_\aposition)
		\in \triples_\aposition(\optimalfunction_\aposition,0)$.
	We will prove that this sequence is built by the algorithm by recurrence on the bags.
	By definition, $\asubsetbag_1 = \bag_1$ and $\partialcost_1 = \cost(\linearorder_1)$,
	therefore the first triple is build by the algorithm. As
	$(\asubsetbag_{\aposition+1},\linearorder_{\aposition+1},\partialcost_{\aposition+1})$
	is compatible with $(\asubsetbag_\aposition,\linearorder_\aposition,\partialcost_\aposition)$,
	then, by definition of compatibility and how the algorithm proceeds, if the algorithm builds
	$(\asubsetbag_\aposition,\linearorder_\aposition,\partialcost_\aposition)$ it will
	build $(\asubsetbag_{\aposition+1},\linearorder_{\aposition+1},\partialcost_{\aposition+1})$
	in the next step.
	This proves the correctness of the algorithm.

	Now we will prove the running time. Without loss of generality, we can assume that
	the path decomposition is a nice path decomposition.
	In a path decomposition with no duplicate bags, there are at most $2n$ bags.
	For each position $\aposition
		\in \intv{2n}$, if $\bag_\aposition$ forgets a vertex $\vvertex$,
	then $\promisingtriples_\aposition$ can be computed by removing $\vvertex$ in each triple
	in $\promisingtriples_{\aposition -1}$ and keeping triples with minimum cost for each fixed
	pair $(\asubsetbag, \linearorder)$. This can be
	done in time $\Oh({|\promisingtriples_{\aposition - 1}|}^2)$.
	If $\bag_\aposition$ introduces a vertex~$\vvertex$, then $\promisingtriples_\aposition$
	can be computed from $\promisingtriples_{\aposition - 1}$ by taking each triple
	$(\asubsetbag, \linearorder, \cost)$ and adding $\vvertex$ in $\linearorder$ at
	every possible position. Computing the new cost can be done in time $\Oh(\width\cdot\log (n^2 \cdot m))$
	and there are $|\promisingtriples_{\aposition + 1}|$ triples to compute.
	The log factor $\log (n^2 \cdot m)$ is the time of performing an addition on the  costs, as the cost of a solution
	can be at most $n^2 \cdot m$.
	By
	Remark~\ref{remark:UpperBoundTriples}, we can bound the size of each~$\promisingtriples_{\aposition}$.
\end{proof}

Now, leveraging on Lemma~\ref{lemma:CompatibleSequenceOne} and Lemma~\ref{lemma:CompatibleSequenceTwo},
we will devise a fixed-parameter tractable algorithm
for \textsc{Diverse-CO} parameterized by solution imperfection, number of solutions, scatteredness, and pathwidth of the cocomparability
graph of the input partial order.
Let $\apartialorder$ be a partial order and $\pathdec = (\bag_1, \bag_2, \ldots, \bag_{\lengthdecomposition})$
be a $\apartialorder$-consistent path decomposition of $\agraph_{\apartialorder}$ of width $\width$.
\begin{definition}
	Let $\aposition\in [\lengthdecomposition]$, and $\optimalfunction:\pairs_{\aposition}\rightarrow \N$. Then, we
	let $\bigtuples_{\aposition}(\optimalfunction,\distanceoptimality)$ be the set of all tuples of the form
	\[( (\asubsetbag^1,\linearorder^1,\partialcost^1), \dots,(\asubsetbag^{\numbersolutions},\linearorder^{\numbersolutions},\partialcost^{\numbersolutions}), \partialdiversity, {(\partialscatteredness_{\{i,j\}})}_{1 \leq i < j \leq \numbersolutions})\]
	where $\partialdiversity\in \intv{\diversityparameter+1}_0$,
	for each $1 \leq i < j \leq \numbersolutions$, $\partialscatteredness_{\{i,j\}} \in \intv{\scatterednessparameter}_0$, and for each $i\in \intv{\numbersolutions}$,
	$(\asubsetbag^{i},\linearorder^{i},\partialcost^{i})$ is a triple in
	$\triples_{\aposition}(\optimalfunction,\distanceoptimality)$.
\end{definition}

Intuitively, ${((\asubsetbag^i,\linearorder^i,\partialcost^i))}_{i\in\intv{\numbersolutions}}$
are $\numbersolutions$ partial linear extensions, $\partialdiversity$ will be the diversity of the
$\numbersolutions$ partial linear extensions and $\partialscatteredness$ will be the distance between
all pair of the $\numbersolutions$ partial linear extensions.

\begin{remark}
	\label{remark:UpperBoundTuples}
	For each $\aposition\in [\lengthdecomposition]$, $\optimalfunction:\pairs_{\aposition}\rightarrow \N$
	and $\distanceoptimality\in \N$, the size of $\bigtuples_{\aposition}(\optimalfunction,\distanceoptimality)$
	is bounded by ${(e\cdot (\distanceoptimality+1)\cdot (\width+1)!)}^{\numbersolutions}
		\cdot \scatterednessparameter^{\numbersolutions^2} \cdot \diversityparameter$.
\end{remark}

For each $\aposition\in [\lengthdecomposition-1]$, each tuple
\[\abigtuple_{\aposition} =
	((\asubsetbag^1_{\aposition},\linearorder^1_{\aposition},\partialcost^1_{\aposition}),
	\dots,(\asubsetbag^{\numbersolutions}_{\aposition},\linearorder^{\numbersolutions}_{\aposition},
	\partialcost^{\numbersolutions}_{\aposition}), \partialdiversity_{\aposition},
	{(\partialscatteredness^\aposition_{\{i,j\}})}_{1 \leq i < j \leq \numbersolutions})\]
in $\bigtuples_{\aposition}$
and each tuple
\[\abigtuple_{\aposition+1} =
	((\asubsetbag^1_{\aposition+1},\linearorder^1_{\aposition+1},\partialcost^1_{\aposition+1}),
	\dots,(\asubsetbag^{\numbersolutions}_{\aposition+1},\linearorder^{\numbersolutions}_{\aposition+1},
	\partialcost^{\numbersolutions}_{\aposition+1}), \partialdiversity_{\aposition+1},
	{(\partialscatteredness^{\aposition+1}_{\{i,j\}})}_{1 \leq i < j \leq \numbersolutions})\]
in $\bigtuples_{\aposition+1}$,
we define the \emph{scatteredness increase table} of the pair
$(\abigtuple_{\aposition},\abigtuple_{\aposition+1})$, denoted by
$\increasescatteredness(\abigtuple_{\aposition},\abigtuple_{\aposition+1})$, as the table
holding the distance increase between each pair of partial linear extensions due to the elements
of $\introduced(\aposition + 1)$.
To compute this increase for each pair $1 \leq i < j \leq \numbersolutions$, we will
compute the increase of adding one vertex and repeat this operation for each element in
$\introduced(\aposition + 1)$. Let $\vvertex \in \introduced(\aposition + 1)$ be a vertex
introduced by $\bag_{\aposition + 1}$ and let $\abigtuple_\vvertex$ be a tuple obtained
by extending~$\abigtuple_{\aposition}$ to include $\vvertex$, we have
\begin{align*}\label{eq:partial_increasescatteredness}
	\increasescatteredness_{\{i,j\}}(\abigtuple_{\aposition},\abigtuple_{\vvertex}) ={}
	 & |\{\uvertex \in \bag_{\aposition} \mid
	(\uvertex \notin \asubsetbag^i_{\aposition+1}  \lor \uvertex <_{\linearorder^i_{\aposition+1}} \vvertex)
	\land \vvertex <_{\linearorder^j_{\aposition+1}} \uvertex\}| +{} \\
	 & |\{\uvertex \in \bag_{\aposition} \mid
	(\uvertex \notin \asubsetbag^j_{\aposition+1}  \lor \uvertex <_{\linearorder^j_{\aposition+1}} \vvertex)
	\land \vvertex <_{\linearorder^i_{\aposition+1}} \uvertex\}|
\end{align*}
where $\uvertex \notin \asubsetbag^i_{\aposition+1}  \lor \uvertex <_{\linearorder^i_{\aposition+1}} \vvertex$
means that $\uvertex$ is smaller than $\vvertex$ in the $i^\text{th}$ tuple,
$\vvertex <_{\linearorder^j_{\aposition+1}} \uvertex$ means that $\uvertex$ is bigger
than $\vvertex$ in the $j^\text{th}$ tuple,
$\uvertex \notin \asubsetbag^j_{\aposition+1}  \lor \uvertex <_{\linearorder^j_{\aposition+1}} \vvertex$
means that $\uvertex$ is smaller than $\vvertex$ in the $j^\text{th}$ tuple, and
$\vvertex <_{\linearorder^i_{\aposition+1}} \uvertex$ means that $\uvertex$ is bigger
than $\vvertex$ in the $i^\text{th}$ tuple.
Now, we define the increase between two bags, given an arbitrary ordering $\vvertex_1, \ldots, \vvertex_{|\introduced(\aposition+1)|}$ of $\introduced(\aposition+1)$: we let $\abigtuple_{\vvertex_1}$ be the tuple obtained
by extending $\abigtuple_{\aposition}$ to include $\vvertex_1$ and $\abigtuple_{\vvertex_t}$ be the
tuple obtained by extending $\abigtuple_{\vvertex_{t-1}}$ to include $\vvertex_t$ and
we have, for each pair $\{i, j\}$,
\begin{equation}\label{eq:increase_scatteredness}
	\increasescatteredness_{\{i,j\}}(\abigtuple_{\aposition},\abigtuple_{\aposition+1}) =
	\increasescatteredness_{\{i,j\}}(\abigtuple_{\aposition},\abigtuple_{\vvertex_1}) +
	\sum_{t = 2}^{|\introduced(\aposition+1)|} \increasescatteredness_{\{i,j\}}(\abigtuple_{\vvertex_{t-1}},\abigtuple_{\vvertex_t}).
\end{equation}
Intuitively, this measures the increase of the distance between each pair of the $\numbersolutions$ partial
solutions from the bag $\aposition$ to the bag $\aposition + 1$.

We define, in a similar way, the \emph{diversity increase} of the pair $(\abigtuple_{\aposition},\abigtuple_{\aposition+1})$,
denoted by writing $\increasepair(\abigtuple_{\aposition},\abigtuple_{\aposition+1})$, as the amount
of diversity due to the elements of $\introduced(\aposition+1)$.
As the diversity is the sum of all pairwise distances between the $\numbersolutions$ partial solutions, we
will use the scatteredness increase table $(\abigtuple_{\aposition},\abigtuple_{\aposition+1})$
to compute the diversity increase.
We define the increase between two bags as
\begin{equation}\label{eq:increase_diversity}
	\increasepair(\abigtuple_{\aposition},\abigtuple_{\aposition+1}) =
	\sum_{1 \leq i < j \leq \numbersolutions} \increasescatteredness_{\{i,j\}}(\abigtuple_{\aposition},\abigtuple_{\aposition+1}).
\end{equation}
Intuitively, this measures the increase of the diversity between the $\numbersolutions$ partial
solutions up to the bag $\aposition$ and the $\numbersolutions$ partial solutions up to
bag $\aposition + 1$.

We say that $\abigtuple_{\aposition}$ is compatible with $\abigtuple_{\aposition+1}$ if
$\partialscatteredness_{\aposition+1} = \min(\partialscatteredness_{\aposition} + \increasescatteredness(\abigtuple_{\aposition},\abigtuple_{\aposition+1}), \scatterednessparameter)$,
$\partialdiversity_{\aposition+1} = \min(\partialdiversity_{\aposition} + \increasepair(\abigtuple_{\aposition},\abigtuple_{\aposition+1}), \diversityparameter)$,
and for each $i\in [\numbersolutions]$,
the triple $(\asubsetbag^i_{\aposition},\linearorder^i_{\aposition},\partialcost^i_{\aposition})$
is compatible with the triple
$(\asubsetbag^i_{\aposition+1},\linearorder^i_{\aposition+1},\partialcost^i_{\aposition+1})$.

A \emph{diversity-compatible sequence} is a sequence of the form
$${\{((\asubsetbag^1_{\aposition},\linearorder^1_{\aposition},\partialcost^1_{\aposition}),
	\dots, (\asubsetbag^{\numbersolutions}_{\aposition},
	\linearorder^{\numbersolutions}_{\aposition},\partialcost^{\numbersolutions}_{\aposition}),
	\partialdiversity_{\aposition},
	{(\partialscatteredness^\aposition_{\{i,j\}})}_{1 \leq i < j \leq \numbersolutions}
	)\}}_{\aposition \in \intv{\lengthdecomposition}},$$
where for each $\aposition\in [\lengthdecomposition-1]$, the tuples at positions $\aposition$ and $\aposition+1$ are compatible.

The next lemma is an analogue of Lemma~\ref{lemma:CompatibleSequenceOne} in the context of solution diversity.

\begin{lemma}
	\label{lemma:CompatibleSequenceDiverseOne}
	Let $\apartialorder\subseteq \baseset\times\baseset$ be a partial order over $\baseset$,
	$\cost:\baseset\times \baseset\to \naturalIntegers$ be a cost function,
	and $\pathdec$ be a $\apartialorder$-consistent path decomposition of the graph $\agraph_{\apartialorder}$.
	Let
	$$\diversitycompatiblesequence = {\{((\asubsetbag^1_{\aposition},\linearorder^1_{\aposition},\partialcost^1_{\aposition}),
		\dots, (\asubsetbag^{\numbersolutions}_{\aposition},
		\linearorder^{\numbersolutions}_{\aposition},\partialcost^{\numbersolutions}_{\aposition},
		\partialdiversity_{\aposition}), {(\partialscatteredness^\aposition_{\{i,j\}})}_{1 \leq i < j \leq \numbersolutions}
		)\}}_{\aposition \in \intv{\lengthdecomposition}}$$
	be a diversity-compatible sequence for $\pathdec$. Then, the following properties can be verified.
	\begin{enumerate}
		\item\label{itemOneCDO} For each $i\in [\numbersolutions]$, the order $\vote_i = \transitiveclosure(\apartialorder \cup \linearorder_1^i\cup \dots \cup \linearorder_{\lengthdecomposition}^i)$ is a linear extension of~$\apartialorder$ of cost $\partialcost_{\lengthdecomposition}^i$.
		\item\label{itemTwoCDO} For each $i,j$ with $1 \leq i < j \leq \numbersolutions$, $\partialscatteredness^\lengthdecomposition_{\{i,j\}} = \min(\KTd(\vote_i, \vote_j), \scatterednessparameter)$.
		\item\label{itemThreeCDO} $\partialdiversity_{\lengthdecomposition} = \min(\ktdiversity(\{\vote_1,\dots,\vote_{\numbersolutions}\}), \diversityparameter)$.
	\end{enumerate}
\end{lemma}

We note that property \ref{itemOneCDO} of Lemma \ref{lemma:CompatibleSequenceDiverseOne} follows from
Lemma \ref{lemma:CompatibleSequenceOne}, property  \ref{itemTwoCDO} follows from
Equation~\ref{eq:increase_diversity}, and property  \ref{itemThreeCDO} from  Equation~\ref{eq:increase_scatteredness}.

The next lemma is an analogue of Lemma~\ref{lemma:CompatibleSequenceDiverseTwo} in the context of solution diversity.

\begin{lemma}%
	\label{lemma:CompatibleSequenceDiverseTwo}
	Let $\apartialorder\subseteq \baseset\times\baseset$ be a partial order over $\baseset$,
	$\cost:\baseset\times \baseset\to \naturalIntegers$ be a cost function,
	and $\pathdec$ be a $\apartialorder$-consistent path decomposition of the graph $\agraph_{\apartialorder}$.
	Let $\vote_1, \ldots, \vote_\numbersolutions$ be $\numbersolutions$ linear extensions of~$\apartialorder$, and
	$$ \diversitycompatiblesequence = {\{((\asubsetbag^1_{\aposition},\linearorder^1_{\aposition},\partialcost^1_{\aposition}),
		\dots, (\asubsetbag^{\numbersolutions}_{\aposition},
		\linearorder^{\numbersolutions}_{\aposition},\partialcost^{\numbersolutions}_{\aposition}),
		\partialdiversity_{\aposition}, {(\partialscatteredness^\aposition_{\{i,j\}})}_{1 \leq i < j \leq \numbersolutions}
		)\}}_{\aposition \in \intv{\lengthdecomposition}}$$ be a sequence satisfying the following conditions.
	\begin{enumerate}
		\item For each position $\aposition \in \intv{\lengthdecomposition}$, and each $i \in \intv{\numbersolutions}$,
		      $\asubsetbag^i_\aposition = \{\vvertex \in \bag_\aposition \mid \vvertex >_{\vote_i} \max_{\vote_i}(\leftset_\aposition)\}$,
		      $\linearorder_\aposition = \vote_i|_{\asubsetbag_\aposition}$,
		      $\partialcost_\aposition = \cost(\vote_i|_{\leftset_\aposition \cup \bag_\aposition})$.
		\item For each $i,j$ with $1 \leq i < j \leq \numbersolutions$,
		      $\partialscatteredness^\aposition_{\{i,j\}} = \min(\KTd(\vote_i|_{\leftset_\aposition}, \vote_j|_{\leftset_\aposition}), \scatterednessparameter)$.
		\item $\partialdiversity_{\aposition} = \min(\ktdiversity(\{\vote_1|_{\leftset_\aposition \cup \bag_\aposition},\dots,\vote_{\numbersolutions}|_{\leftset_\aposition \cup \bag_\aposition}\}), \diversityparameter)$.
	\end{enumerate}
	Then,
	$\diversitycompatiblesequence$ is a diversity-compatible sequence for $\pathdec$.
\end{lemma}

Intuitively, what those lemmas say is that in order to construct a set of $\numbersolutions$ solutions
for an instance $(\apartialorder,\cost,\numbersolutions,\distanceoptimality,\diversityparameter,\scatterednessparameter)$ of \textsc{Diverse-CO}, all one needs to
do is to construct $\numbersolutions$ compatible sequences in parallel, by processing the given path decomposition
from left to right, while using an additional register to keep track of the overall diversity at each time step
and all the pairwise distances.
In the same way that Lemma~\ref{lemma:CompatibleSequenceOne} and Lemma~\ref{lemma:CompatibleSequenceTwo} yield an \FPT\ dynamic programming algorithm parameterized
by pathwidth for computing a single solution of an instance of {\sc CO} (Theorem~\ref{theorem:SolutionConstructionOne}),
Lemma~\ref{lemma:CompatibleSequenceDiverseOne} and Lemma~\ref{lemma:CompatibleSequenceDiverseTwo} yield a dynamic programming algorithm to compute a diverse set of solutions,
in case it exists, parameterized by cost of solution, number of solutions and pathwidth (Theorem~\ref{theorem:SolutionConstructionDiverse}).

\begin{theorem}
	\label{theorem:SolutionConstructionDiverse}
	Let $\apartialorder\subseteq \baseset\times \baseset$, let $n = |\baseset|$ and $\width$ be the pathwidth
	of the cocomparability graph of $\apartialorder$,
	and $\cost:\baseset\times \baseset\rightarrow \intv{m}_0$
	be a cost function.
	Then, one can determine whether $\apartialorder$ admits $\numbersolutions$ linear
	extensions $\vote_1,\dots,\vote_{\numbersolutions}$
	at distance at most $\distanceoptimality$ from the optimum,
	diversity at least $\diversityparameter$, and scatteredness at least $\scatterednessparameter$ in time
	$\Oh\left({(\width! \cdot  \distanceoptimality)}^{\Oh(\numbersolutions)}
		\cdot \scatterednessparameter^{\numbersolutions^2} \cdot \diversityparameter
		\cdot n \cdot \log (n^2 \cdot m) \right)$.
\end{theorem}
\begin{proof}
	This theorem is similar to Theorem~\ref{theorem:SolutionConstructionOne}. Here
	we build $\numbersolutions$ linear extension in parallel and we incrementally compute
	the diversity between the $\numbersolutions$ solutions.
	The main difference is the computation of the diversity.
	Using Equation~\ref{eq:increase_scatteredness} and {}\ref{eq:increase_diversity},
	one can compute the increase of the pairwise distances and diversity in time
	$\Oh(\numbersolutions^2\cdot \width \cdot \log (n^2 \cdot m))$ for each vertex and each
	tuple.
	The log factor $\log (n^2 \cdot m)$ again comes from performing an addition of costs, as the cost of a solution
	can be at most $n^2 \cdot m$.
\end{proof}

By combining Theorem~\ref{theorem:SolutionConstructionOne} with our reduction from {\sc KRA} to {\sc CO},
we have an \FPT\ algorithm for {\sc KRA}, parameterized by solution imperfection, number of solutions, scatteredness,
and unanimity
width (Corollary~\ref{corollary:DiverseGKRA}).

\begin{corollary}
	\label{corollary:DiverseGKRA}
	Let~$\votes$   be a list of $m$ partial votes over a set of $n$ candidates $\candidates$. Let $\width$ be the unanimity width of $\votes$.
Given~$\votes$ and non-negative integers $\distanceoptimality$, $\numbersolutions$, $\scatterednessparameter$
	and $\diversityparameter$, 
	one can determine in time
	\[\Oh\left({(\width! \cdot \distanceoptimality)}^{\Oh(\numbersolutions)}
		\cdot \scatterednessparameter^{\numbersolutions^2} \cdot \diversityparameter \cdot n
		\cdot \log (n^2 \cdot m) \right)\]
	whether there is a set
	$\rankingset = \{\vote_1,\dots,\vote_{\numbersolutions}\}$ of $\numbersolutions$ linear orders
	on~$\candidates$ such that the Kemeny score for each order $\vote_i$ is at distance at most
	$\distanceoptimality$ of the optimum, and we find that $\ktdiversity(\rankingset)\geq \diversityparameter$
	and that scatteredness is  at least $\scatterednessparameter$.
\end{corollary}

Corollary~\ref{corollary:DiverseGKRA} is our most general result that combines all the
parameters, but not all applications need all parameters. Therefore, we will now derive
some special cases.

\begin{corollary}
	\label{corollary:no_scatter}Let~$\votes$   be a list of $m$ partial votes over a set of $n$ candidates $\candidates$. Let $\width$ be the unanimity width of $\votes$.
Given~$\votes$ and non-negative integers $\distanceoptimality$, $\numbersolutions$
	and $\diversityparameter$, 
	one can determine in time
	\[\Oh\left({(\width! \cdot \distanceoptimality)}^{\Oh(\numbersolutions)}
		\cdot \diversityparameter \cdot n
		\cdot \log (n^2 \cdot m) \right)\]
	whether there is a set
	$\rankingset = \{\vote_1,\dots,\vote_{\numbersolutions}\}$ of $\numbersolutions$ linear orders
	on~$\candidates$ such that the Kemeny score for each order $\vote_i$ is at distance at most
	$\distanceoptimality$ of the optimum, and we find that $\ktdiversity(\rankingset)\geq \diversityparameter$.
\end{corollary}

Namely, by our formulation of $R$ as a set, we implicitly have the requirement $s\geq 1$.

With our algorithm, it is possible to check if there exists $\numbersolutions$ different
optimal solutions. We can do this by setting $\distanceoptimality = 0$, $\scatterednessparameter = 1$
and $\diversityparameter = 0$ and we get the following corollary.

\begin{corollary}
	\label{corollary:multple_opt}Let~$\votes$   be a list of $m$ partial votes over a set of $n$ candidates $\candidates$. Let $\width$ be the unanimity width of $\votes$.
Given~$\votes$ and non-negative integers $\distanceoptimality$, $\numbersolutions$, $\scatterednessparameter$
	and $\diversityparameter$, 
	one can determine in time
	\[\Oh\left({(\width!)}^{\Oh(\numbersolutions)}
		\cdot 2^{\numbersolutions^2} \cdot n
		\cdot \log (n^2 \cdot m) \right)\]
	whether there is a set
	$\rankingset = \{\vote_1,\dots,\vote_{\numbersolutions}\}$ of $\numbersolutions$
    different optimal linear orders on~$\candidates$.
\end{corollary}

To get some insights of the structure of the solution space, one can ask for a set
of $\numbersolutions$ solutions of cost at most $\distanceoptimality$ from the minimum cost
with maximum diversity. This is not strictly a consequence of Corollary~\ref{corollary:DiverseGKRA}
but the same algorithm can be used to answer this question.
In our algorithm, the value $\partialdiversity$ is used to compute the
diversity of a partial solution up to
$\diversityparameter$, but if the diversity is bigger than
$\diversityparameter$, we just remember $\diversityparameter$. Then if we set $\diversityparameter =
\numbersolutions \cdot n^2$, which is an upper bound of the maximum diversity of $\numbersolutions$ solutions,
at the end of the algorithm, we will have a set of possible solutions with their exact diversities.
From this, we can select the one with the biggest diversity. Therefore, we have the following
result.

\begin{corollary}
	\label{corollary:max_diversity}
Let~$\votes$   be a list of $m$ partial votes over a set of $n$ candidates $\candidates$. Let $\width$ be the unanimity width of $\votes$.
Given~$\votes$ and non-negative integers $\distanceoptimality$ and $\numbersolutions$,  
	one can compute in time
	\[\Oh\left({(\width! \cdot \distanceoptimality)}^{\Oh(\numbersolutions)}
		\cdot \numbersolutions \cdot n^3
		\cdot \log (n^2 \cdot m) \right)\]
	a set $\rankingset = \{\vote_1,\dots,\vote_{\numbersolutions}\}$ of $\numbersolutions$ linear orders
	on~$\candidates$ such that the Kemeny score for each order $\vote_i$ is at distance at most
	$\distanceoptimality$ of the optimum and such that $\ktdiversity(\rankingset)$ is maximal.
\end{corollary}

%% file: 06-SubexponentialPCOSafe.tex
\section{Sub-Exponential Time Algorithm for \textsc{PCO}}
\label{section:sub-expo-PCO}

For special cases of \textsc{PCO},
such as those arising from \textsc{KRA} or from the graph-drawing problem \textsc{OSCM},
single-exponential sub-exponential time algorithms have been known, i.e., algorithms
with running times of the form $\Oh^*(2^{\Oh(\sqrt{k})})$. In contrast, for the more general problem of \textsc{PCO},
only algorithms with running time $\Oh^*(k^{\sqrt k})$ were known before, where~$k$ is the cost parameter~\cite{Feretal2014}.
Here, we prove that \textsc{PCO} also admits algorithms of the form $\Oh^*(2^{\Oh(\sqrt{k})})$,
by making use of several structural insights for cocomparability graphs. More precisely,
we prove the following theorem.

\begin{theorem}\label{theorem:main}
	Given a partial order $\partialorder\subseteq\baseset\times\baseset$ and
	a cost function $\cost:\baseset\times \baseset \to \N$,
	one can solve a \textsc{PCO}  instance $(\partialorder,\cost,\parameter)$  in time
	$|\baseset| \cdot 2^{\Oh(\sqrt{\parameter})} + \Oh(|\baseset|^2\cdot \log(\parameter))$.
\end{theorem}

The remainder of this section is dedicated to the proof of Theorem \ref{theorem:main}.

The treewidth of a graph is another structural parameter that quantifies
how close the graph is to being a forest (i.e., a graph without
cycles). In general graphs, this parameter is more expressive than pathwidth in the
sense that any graph of pathwidth $k$ has also treewidth $k$, but
there exist graphs of treewidth $k$ whose pathwidth is unbounded.
Nevertheless, in the class of cocomparability graphs, treewidth and
pathwidth coincide.

\begin{lemma}\label{tree_pathwidth} [\cite[Theorem 1.2]{habib1994treewidth}]
	Let $\graph$ be a cocomparability graph. Then, $\pw(\graph) = \tw(\graph)$.
\end{lemma}

Let $\graph = (\vertexset, \edgeset)$ be a graph and $S \subseteq \vertexset$ be a set
of vertices. We call $\graph_S = (S, \edgeset \cap (S \times S))$ the subgraph of $\graph$
\emph{induced} by $S$. Let $H$ be a graph, we say that $\graph$ contains $H$ as an
induced subgraph if there exist a subset of vertices $S \subseteq \vertexset$ of $\graph$
such that $H$ is isomorphic to $\graph_S$. We say that $\graph$ is $H$-free 
if $\graph$ does not contain $H$ as an induced subgraph.

Let $t \geq 3$ be an integer, we write $C_t$ for the cycle on $t$ vertices and
$C_{\geq t} = \{C_k \mid k \geq t\}$. We say that $\graph$ is $C_{\geq t}$-free
if $\graph$ excludes all $C_k$  as an induced subgraph for any $k \geq t$. 

\begin{lemma}\label{cycles_exclusion}
	Cocomparability graphs are $C_{\geq 5}$-free.
\end{lemma}

This fact is known but
not that easy to track down in the literature.
We refer to \cite{Gho64,gilmore_hoffman_1964,Gal67}, which contain corresponding results on comparability graphs.
We also refer to the textbook of Trotter~\cite{Trotter92}.
To keep the paper self-contained, we present a short self-contained proof of the fact
that cocomparability graphs are $C_k$ free for
$k\geq 5$.  As a key notion, we consider bad triples in cocomparability orders.
A \emph{cocomparability order} of a cocomparability graph $G=(V,E)$ is a bijection
$\sigma : V \to \{1, \ldots, |V|\}$ that linearly extends a transitive
orientation $\apartialorder$ of $G$, meaning that $(x,y)\in \apartialorder$
implies $\sigma(x)<\sigma(y)$. Given a graph $\agraph$ and a linear order $\sigma$ of its vertices,
a \emph{bad triple} is three vertices $x$, $y$, $z$ so that
$\sigma(x) < \sigma(y) < \sigma(z)$, $xy \notin E$, $yz \notin E$, and $xz \in E$.
Notice  that if $\agraph$ is a cocomparability graph and $\sigma$ a cocomparability order
then $(G,\sigma)$ has no bad triple.
Namely, if $x,y$ and $y,z$ are comparable in some partial order $\apartialorder$ with  $\sigma(x) < \sigma(y) < \sigma(z)$, then (as $\sigma$ extends $\apartialorder$) whenever $x<_\apartialorder y$ and $y<_\apartialorder z$
then $x<_\apartialorder  z$ is enforced by transitivity, ruling out a bad triple.

\begin{lemma}
	A cycle $C_k$ on $k \geq 5$ vertices is not a cocomparability graph.
\end{lemma}

\begin{proof}
	Suppose for contradiction that $\sigma : V(C_k) \rightarrow \{1, \ldots, k\}$ is a cocomparability order of $C_k$ with edge set $E_k$. Define $\sigma^{-1}(i)$ to be the vertex $x$ in $C_k$ so that $\sigma(x) = i$. Let $u = \sigma^{-1}(1)$ and $v = \sigma^{-1}(k)$.
	If $uv \in E_k$, let $x$ be any vertex non-adjacent to both $u$ and $v$ (such a vertex exists since $k \geq 5$), we have that $u$, $x$, $v$ is a bad triple. We conclude that  $uv \notin E_k$.
	Let $P$ and $Q$ be the two paths that connect $u$ and $v$ in $C_k$, excluding $u,v$. Without loss of generality, $|V(P)| \geq |V(Q)|$. Since $k \geq 5$, we have that the path $P$ contains an edge $pq$ so that $\{p,q\} \cap \{u,v\} = \emptyset$. Without loss of generality, $\sigma(p) < \sigma(q)$.

	First, suppose that there is a vertex $\ell \in V(Q)$ so that $\sigma(p) < \sigma(\ell) < \sigma(q)$. Then $p, \ell, q$ is a bad triple. So such a vertex cannot exist.
	It follows that some edge $ab$ of $E(Q\cup\{u,v\})$ is such that $\sigma(a) < \sigma(p)$ and $\sigma(q) <  \sigma(b)$. Since the path $Q$ has at least one internal vertex, we have that $|\{a, b\} \cap \{u, v\}| \leq 1$, and therefore (since each of $u$ and $v$ has at most one neighbor in $\{p, q\}$) we have $|N(\{a, b\}) \cap \{p, q\}| \leq 1$. However, if $p \notin N(\{a, b\})$ then  $a, p, b$ is a bad triple. If $q \notin N(\{a, b\})$ then  $a, q, b$ is a bad triple. But $|N(\{a, b\}) \cap \{p, q\}| \leq 1$ implies that one of the two former cases must hold, yielding the desired contradiction.
\end{proof}

Finally, as the proof is based on the notion of bad triples and as bad triples in induced
subgraphs are also bad triples in the whole graph, this proves Lemma~\ref{cycles_exclusion}.

The following statement is put forward in \cite[Theorem~1.5]{DBLP:conf/soda/ChudnovskyPPT20}.

\begin{lemma}[\cite{DBLP:conf/soda/ChudnovskyPPT20} (Theorem~1.5)]\label{chudnovsky}
	A $C_{\geq t}$-free graph with maximum degree $\Delta$ has treewidth bounded by
	$\Oh(t\cdot \Delta)$. Furthermore, a tree decomposition of this width can be
	computed in polynomial time.
\end{lemma}
%

\noindent
From Lemma \ref{chudnovsky}, we can prove the following lemma.

\begin{lemma}\label{treewidth_edges}
	Let $\graph$ be a $C_{\geq 5}$-free graph, let $m$ be the number of edges of $\graph$.
	Then, we have $m  = \Omega(\tw(\graph)^2)$.
\end{lemma}
\begin{proof}
	A graph on $m$ edges has at most $2\sqrt{m}$ vertices
	of degree at least $\sqrt{m}$. Let $\graph'$ be the graph obtained from $\graph$ by removing all vertices
	of degree at least $\sqrt{m}$. Since $\agraph$ is $C_{\geq 5}$-free, so is $\agraph'$.
	Therefore, by applying Lemma~\ref{chudnovsky}, we get that $\tw(\graph')  = \Oh(\sqrt{m})$.
	The removal of a vertex reduces the treewidth by at most $1$. Hence, we have
	$\tw(\graph) - 2\sqrt{m} \leq \tw(\graph') =  \Oh(\sqrt{m})$. This implies
	that $\tw(G) = \Oh(\sqrt{m})$, or conversely, that $m=\Omega(\tw(\graph)^2)$.
\end{proof}

\noindent
By combining Lemma \ref{treewidth_edges} with Lemma \ref{tree_pathwidth}, we get: 

\begin{lemma}
	\label{lemma:pathwidth_edge_cocomparability}
	Let $\graph$ be a cocomparability graph and let $m$ be the number of edges of $\graph$.
	Then, $m = \Omega(\pw(\graph)^2)$.
\end{lemma}

Now, in a \textsc{PCO} instance, each edge contributes at least $1$ to the cost of any solution.
Therefore, if a solution has cost at most $k$, then the cocomparability graph of the input partial
order can have at most $k$ edges. Therefore, this observation, together with
Lemma \ref{lemma:pathwidth_edge_cocomparability} yields the following lemma.

\begin{lemma}
	\label{lemma:PCO}
	Let $(\apartialorder,\cost,k)$ be an YES-instance of \textsc{PCO}. Then,
	$\pw(\graph_{\apartialorder})  = \Oh(\sqrt{k})$.
\end{lemma}

To get the running time of Theorem~\ref{theorem:main}, we need to either compute a $\apartialorder$-consistent
path decomposition of width at most $\Oh(\sqrt{\parameter})$, or to trigger an early rejection.
For this, we will use the following lemma which is based on Lemma~\ref{lemma:ConstructionDecomposition}.

\begin{lemma}
	\label{lemma:ConstructionVsNoInstance}
	There is a polynomial-time algorithm that takes an instance of $(\apartialorder,\cost,k)$ of \textsc{PCO} as input,
	and either constructs a $\apartialorder$-consistent path decomposition of the graph $\graph_{\apartialorder}$ of
	width $\Oh(\sqrt{k})$, or determines that this instance is a NO-instance.
\end{lemma}

In Section \ref{section:DPCO}, in order to define our algorithm for {\sc Diverse-CO}, we first devised a
simpler algorithm for the single-solution version of {\sc CO}, that could be used as a building block for the diverse algorithm.
It turns out that if our only goal is to solve the single-solution version of  {\sc CO}, then the basic algorithm developed in
Section \ref{section:DPCO} can be optimized, to become a single-exponential time algorithm parameterized by the
pathwidth of the cocomparability graph of the input order. More precisely, we have the following lemma.

\begin{lemma}[\cite{AFOW2020} (Theorem 1)]
	\label{lemma:SingleExponentialPCO}
	Given an instances $(\apartialorder,\cost,k)$ of \textsc{CO} and a $\apartialorder$-consistent path decomposition
	$\pathdec$ of the graph $\graph_{\apartialorder}$, one can solve this instance in time
	$|\baseset|\cdot 2^{\Oh(\pw(\agraph_{\apartialorder}))}\cdot \log(\parameter) + \Oh(|\baseset|^2\cdot \log(\parameter))$.
\end{lemma}

Now we are ready to prove the statement of Theorem \ref{theorem:main}. Given an instance $(\apartialorder,\cost,k)$ of
\textsc{PCO}, we apply the algorithm stated in Lemma~\ref{lemma:ConstructionVsNoInstance}. This algorithm either determines
that the instance is a NO-instance, or constructs a $\apartialorder$-consistent path decomposition $\pathdec$ of $\agraph_{\apartialorder}$
of width $\Oh(\sqrt{k})$. In the first case, we are done and simply answer NO. Otherwise, we give both the instance
$(\apartialorder,\cost,k)$ and the decomposition $\pathdec$ to the algorithm stated in Lemma \ref{lemma:SingleExponentialPCO}
to determine in time $|\baseset| \cdot 2^{\Oh(\sqrt{\parameter})} + {\Oh(|\baseset|^2\cdot \log(\parameter))}$ whether $(\apartialorder,\cost,k)$ is a YES- or a NO-instance
of {\sc PCO}. I\longversion{t is worth noting that i}n case this is a YES-instance, the algorithm also constructs a
linear extension of $\apartialorder$ of cost at most~$k$.
This concludes the proof of Theorem \ref{theorem:main}.

%% file: 07-Conclusion.tex
\section{Conclusion}
\label{section:conclusion}

In this work, we have addressed the \textsc{Kemeny Rank Aggregation} problem,
one of the most central problems in the theory of social choice, 
from the perspective of diversity of solutions and parameterized complexity
theory. We have devised a fixed parameter tractable algorithm for the diverse version of 
\textsc{KRA} with partially ordered votes where parameters are the 
solution imperfection, the number of solutions, the scatteredness and the unanimity width of the set of votes.
As a by-product of our work, we have introduced new parameterized algorithms for 
problems in order theory that are of independent interest. In particular, 
we developed parameterized algorithms for the diverse version of the 
\textsc{Completion of an Ordering} problem (\textsc{CO}). Furthermore, we have developed 
a new sub-exponential time algorithm for the \textsc{\pbPCOname} problem (\textsc{PCO}), a restriction
of {\sc CO} which can be used as a base to solve combinatorial problems in a 
wide variety of fields, such as artificial intelligence, graph drawing, computational
biology, etc. We believe that both our new sub-exponential time algorithm for 
finding single solutions for \textsc{PCO} and our algorithm for finding diverse solutions 
in \textsc{CO} have a very positive impact on the study of these and related computational problems 
in neighboring fields.

%% file: 08-Acknowledgements.tex
\newpage
\section*{Acknowledgements} 
Emmanuel Arrighi acknowledges support from the Research Council of Norway (Grant no. 274526) and from IS-DAAD (Grant no. 309319)). Henning Fernau acknowledges support from 
DAAD PPP (Grant no. 57525246)). The research of Daniel Lokshtanov is supported by BFS award 2018302 and NSF award CCF-2008838. Mateus de Oliveira Oliveira acknowledges support from 
the Research Council of Norway (Grant no. 288761), IS-DAAD (Grant no. 309319) and Sigma2 Network (NN9535K). Petra Wolf acknowledges support from DFG project FE 560/9-1 and DAAD PPP (Grant no. 57525246).